\documentclass[runningheads]{llncs}
\usepackage[T1]{fontenc}
\usepackage{graphicx}
\usepackage{booktabs}
\usepackage[misc]{ifsym}
\usepackage{amsfonts}	
\usepackage{amsmath}
\usepackage{duckuments}						
\usepackage[linesnumbered,ruled,vlined,noresetcount]{algorithm2e}
\usepackage{caption}
\usepackage{subcaption}
\usepackage{colortbl}
\usepackage{xcolor}

\usepackage{hyperref}
\usepackage{graphicx}
\usepackage{multirow}

\SetKwInput{KwInput}{Input}                
\SetKwInput{KwOutput}{Output}
\usepackage{balance}
\definecolor{tabred}{RGB}{214,39,40}
\definecolor{tabblue}{RGB}{31,119,180}
\definecolor{blue}{RGB}{0,0,156}
\definecolor{ForestGreen}{RGB}{0,128,0}
\definecolor{cyan8}{RGB}{48, 195, 209}  
\definecolor{gray8}{RGB}{130, 130, 130} 
\definecolor{olive8}{RGB}{151, 163, 68}  
\definecolor{purple8}{RGB}{160, 120, 190} 
\definecolor{deepyellow}{RGB}{240, 180, 40} 

\newcommand{\corr}{(\Letter)}

\usepackage{mwe}

\begin{document}

\title{Cross-attentive Cohesive Subgraph Embedding to Mitigate Oversquashing in GNNs}

\titlerunning{Cross-attentive Cohesive Subgraph Embedding}



\author{Tanvir Hossain\inst{1} \corr \and
Muhammad Ifte Khairul Islam\inst{1} \and Lilia Chebbah\inst{1} \and 
Charles Fanning\inst{2} \and Esra Akbas\inst{1}
}



\authorrunning{Hossain et al.
}



\institute{Department of Computer Science, Georgia State University, Atlanta, GA 30302, USA \and Department of Data Science and Analytics, Kennesaw State
University, 1000 Chastain Road, Kennesaw, GA 30144, USA \\
\email{\{thossain5, mislam29, lchebbah1\}@student.gsu.edu, cfannin8@students.kennesaw.edu, eakbas1@gsu.edu}
}


\maketitle              
\newcommand{\cacose}{\texttt{CaCoSE}}

\begin{abstract}
Graph neural networks (GNNs) have achieved strong performance across various real-world domains. Nevertheless, they suffer from oversquashing, where long-range information is distorted as it is compressed through limited message-passing pathways. This bottleneck limits their ability to capture essential global context and decreases their performance, particularly in dense and heterophilic regions of graphs. To address this issue, we propose a novel graph learning framework that enriches node embeddings via cross-attentive cohesive subgraph representations to mitigate the impact of excessive long-range dependencies. This framework enhances the node representation by emphasizing cohesive structure in long-range information but removing noisy or irrelevant connections. Additionally, It preserves essential global context without overloading the narrow bottlenecked channels, thereby enhancing message aggregation across vertices. Extensive experiments on multiple benchmark datasets demonstrate that our model achieves consistent improvements in classification accuracy over standard baseline methods.

\keywords{Oversquashing \and Graph Decomposition \and Subgraph Pooling \and Cross-attention \and Graph Representation Learning}
\end{abstract}

\section{Introduction}

Oversquashing~\cite{alon2020bottleneck} is a critical limitation of GNNs that limits their applications in diverse graph domains. It occurs when GNNs capture local signals effectively but  struggle to model long-range dependencies. For growing size of the $n$-hop neighborhood with respect to path distance leads to excessive message propagation through limited edges and narrow channels. Therefore, meaningful multi-hop information becomes bottlenecked, compressed or lost, thereby leading to oversquashing problem of GNNs. 

Many studies have been proposed to understand and quantify oversquashing using node sensitivity~\cite{topping2021understanding}, effective resistance~\cite{black2023understanding}, and commute time~\cite{di2023over}. The most common approach to alleviate this issue is rewiring, adding edges to reduce bottlenecks, but often causes high complexity~\cite{karhadkar2022fosr,topping2021understanding,nguyen2023revisiting,chen2022redundancy}. Moreover, the majority of these methods rely on computationally intensive techniques, such as spectral decomposition \cite{karhadkar2022fosr} and optimal transport \cite{nguyen2023revisiting}.

Real-world networks often preserve critical information within tightly connected dense regions. In presence of weak ties among regions, uniform message aggregation across these neighborhoods causes oversquashing. Dense regions dominate message passing, forcing long-range dependencies through narrow bottlenecks. Graph Decomposition offers a potential remedy by isolating structurally significant regions. However, the nondeterministic nature of traditional algorithms \cite{blondel2008fast,karypis1997metis} often results in suboptimal graph partitions. To address this issue, cohesion-aware graph decomposition operates hierarchically, preserving subgraphs' semantic concerning the original network. Besides, these algorithms can be executed in parallel~\cite{chen2024parallel,chu2022hierarchical}, causing negligible decomposition overhead. 

Furthermore, applying GNNs to these subgraphs allows capturing considerable task-relevant information for vertices in downstream graph analytics. We conduct a pilot study (in section \ref{susbsec:pilot_study}) that illustrates cohesion-sensitive graph partition shortens long-range dependencies while preserving homophily within subgraphs. Furthermore, pooling over these subgraphs enhances homophilic characteristics. Preserving homophily facilitates effective representation learning in graph-based models \cite{gu2023homophily}. Nonetheless, focusing solely on the densely connected regions may compromise long-range connectivity. Cross-attention~\cite{vaswani2017attention} addresses this limitation by enabling interactions between distant nodes, connecting relevant representations even when they are positionally unaligned. 


\noindent\textbf{Our Work.} We propose a novel graph learning framework, that utilizes \textbf{C}ross-\textbf{a}ttentive \textbf{Co}hesive \textbf{S}ubgraph  \textbf{E}mbeddings $(\cacose)$ to alleviate  oversquashing issue in GNNs. This $\cacose$ framework begins by a leveraging $k$-core~\cite{malliaros2020core} algorithm to compute the edges' cohesiveness scores based on their membership in $k$-level regions. Using these scores, it constructs edge-induced subgraphs guided by structural closure and feed them into GNNs to learn node representations. 

After getting nodes representations within each subgraph, it applies graph pooling, followed by cross-subgraph attention mechanism. By graph pooling, $\cacose$ selectively filters out noisy or irrelevant connections while preserving task-relevant structures and critical global information. Moreover, the attention module acts a learnable interfaces between vertices to capture their mutual relations that provides more expressivity in representations. 

Finally, the enriched subgraphs' representations are combined with the embeddings of nodes present in those subgraphs to capture long range global information. This enhances message aggregation and maintains essential local and global connectivity in representations, even in large-scale heterophilic networks. Throughout the paper, we refer to the $\cacose$ framework as our model for simplicity. The contributions of our model is as follows: 

\begin{itemize}
    \item The pilot study opens up a new viewpoint to understand the reasons for oversquashing prior to the experiment. Closure-guided $k$-core decomposition benefits cohesive awareness among vertices, which removes noisy pathways and provides essential locality for GNNs operations. Moreover, pooling enhances network's homophily to attain meaningful subgraph representations.
    \item The cross-subgraph attention  mechanism encodes essential global relations among all subgraphs. Merging node representations with enriched subgraph embeddings maintains both local and global connectivity, helping to alleviate the oversquahing problem in GNNs.
    \item In an extensive experiment on multiple datasets, $\cacose$ outperforms the standard baselines in both node \textbf{(NC)} and graph classification \textbf{(GC)} tasks.
\end{itemize}


\section{Methodology}\label{sec:methodology}
In this section, we first describe the primary components of $\cacose$. Then, we present a pilot study that motivates the design choices of our model. Finally, we provide details of $\cacose$ architecture and describe the representation learning procedure to address the oversquashing problem in GNNs. 

\subsection{Preliminaries}
\subsubsection{Oversquashing.}
One of the major drawbacks of GNNs that occurs when information is severely bottlenecked due to its failure in useful message passing. According to~\cite{topping2021understanding}, when a node $t$ that is connected with another node $s$ by an $r$-hop distance, the Jacobian operation $\delta{h_{t}^{(r+1)}/\delta x_{s}}$ denotes the change of the feature vector $x_{s}$’s impact on the $(r+1){st}$ layer's output of $h_{t}^{(r+1)}$. The quantification is observed from the absolute value $|\delta{h_{t}^{(r+1)}/\delta x_{s}}|$ of the Jacobian, where a negligible value indicates limited information propagation or oversquashing.
In addition, as presented in~\cite{black2023understanding}, oversquashing can be  associated with effective resistance between two node pairs. As low as the effective resistance is, the two nodes have more influence on each other during GNNs' operations.

\subsubsection{Problem Statement.}
Our study addresses the challenge of oversquashing by introducing a closure-aware cohesive subgraph decomposition framework along with graph pooling and a cross-subgraph attention mechanism. The primary objective is to enhance graph component features' expressivity in downstream analytics tasks. In a formal sense, a graph is denoted as $G = (V, E, X)$, where $V$ is the vertex set, $E$ is the edge set and $X \in \mathbb{R}^{N_{v} \times d}$ represents the initial feature vector of the vertices. The number of nodes is $|V| = N_{v}$ and $d$ denotes the dimension of the nodes' features. The graph will be partitioned into cohesion-centric subgraphs, such as $S = \{S_{1}, S_{2},..S_{k_{max}}\}$ where each $S_{k}\subseteq G$ corresponds to cohesion level $k = \{1,2,.., k_{max}\}$. We apply the proposed $\cacose$ model to these subgraphs. For node classification, the model trained on $D_{v} = ( G, \mathbb{X}, Y_v)$ to learn the mapping $f_{v}: \mathbb{X} \rightarrow Y_v$. Besides, for graph classification, it is trained on $D_{G} = (\mathbb{G}, Y_{G})$ to learn $f_G: \mathbb{G} \rightarrow Y_{G}$. In both cases, the learned functions leverage the reduced graph complexity achieved through the $\cacose$ framework. 

\subsubsection{Graph Neural Network.}
Graph neural networks (GNNs) embedded structural information in the graph via message propagation into node and graph representations. In graph convolution network (GCN), message propagation rule is defined as:
\begin{equation}
     H^{(l+1)} = \sigma(\tilde {D}^{-\frac{1}{2}}\tilde{A}\tilde {D}^{-\frac{1}{2}}H^{(l)}\Theta^{(l)})
    \label{eq:1}   
\end{equation}
where, $\tilde{A} = A +I_{N_v}$ denotes the adjacency matrix with the self-connections, $\tilde{D}_{ii} = \sum_{j}{\tilde{A}_{ij}}$ degree matrix and $\Theta \in \mathbb{R}^{d \times h}$ learnable weight matrix for $h$-dimensional embeddings. $\sigma(.)$ indicates the $ReLU(x) = max(0, x)$ activation. $H^{(l)} \in \mathbb{R}^{N_v \times h}$ denotes the $l^{th}$ layer’s nodes’ embeddings matrix for $h$-dimension where $H^{(0)} = X$. 

\subsubsection{Self-Attention Graph Pooling.} \label{subsubsec:SAGPool}
Graph pooling captures the entire graph's representation through compression. SAGPool~\cite{lee2019self} facilitates self-attentive graph learning through selecting impactful neighbors of vertices. It utilizes the (GCN) to measure the self-attention scores.
\begin{equation}
   Attn_{v} = \sigma(\tilde {D}^{-\frac{1}{2}}\tilde{A}\tilde {D}^{-\frac{1}{2}}X\Theta_{att}) 
    \label{eq:2} 
\end{equation}
Where $Attn_v \in \mathbb{R}^{N_v \times 1}$ is calculated from $\tilde {A} = A + I_N,~X$ and a learnable parameters' matrix $\Theta_{att}$.
Then utilizing the pooling ratio $(PR)$ it selects the top-$k$ node indices, $idx = topk(Attn_v, \lceil (PR)N_v \rceil)$. Next, only considers the top-$k$ vertices and their connections, $Attn_{mask} = Attn_{v}[{idx}]$. After the computation of $top$-k node indexes, the graph pooling of nodes’ features is measured as $X_{out} = X_{idx,:}. \odot Attn_{mask}$ and $A_{out} = A_{idx, idx}$. Finally, through a READOUT function the representation of the entire graph  is computed as $Z_{G} = READOUT{(X_{out})}$ . Here, READOUT is the global pooling function: sum, mean, or other advanced learnable aggregator.

\subsubsection{k-core Decomposition.}
Cohesive subgraph decomposition is instrumental to determine the intended graph regions. The decomposition is performed by iteratively peeling away nodes whose degrees fall below the specified threshold: nodes with degree less than $k$ are removed. 

In particular, if the graph has no isolated node, the original graph $G$ can be denoted as $1$-core. As indicated by hierarchy, $G_{k_{max}} \subseteq \dots G_{3} \subseteq G_{2} \subseteq G_{1}$, where, $k \in \{1,2, \dots k_{max}\}$. A $k$-core subgraph is defined as - 
\begin{definition}
    
\label{def:1}
(\textbf{$k$-core}): For a given $k \ge 1$, the $k$-core subgraph is represented from the graph $G$ when each node in the subgraph belongs to the same $(k)$ or more neighbors such that $|N(v)| \ge k$, where $N(v)$ represents the number of neighbors of node $v \in G_{k}$.
\end{definition}
\begin{figure*}[h!]
    \centering
    \begin{subfigure}{0.45\textwidth}
        \centering
        \includegraphics[width=\textwidth]{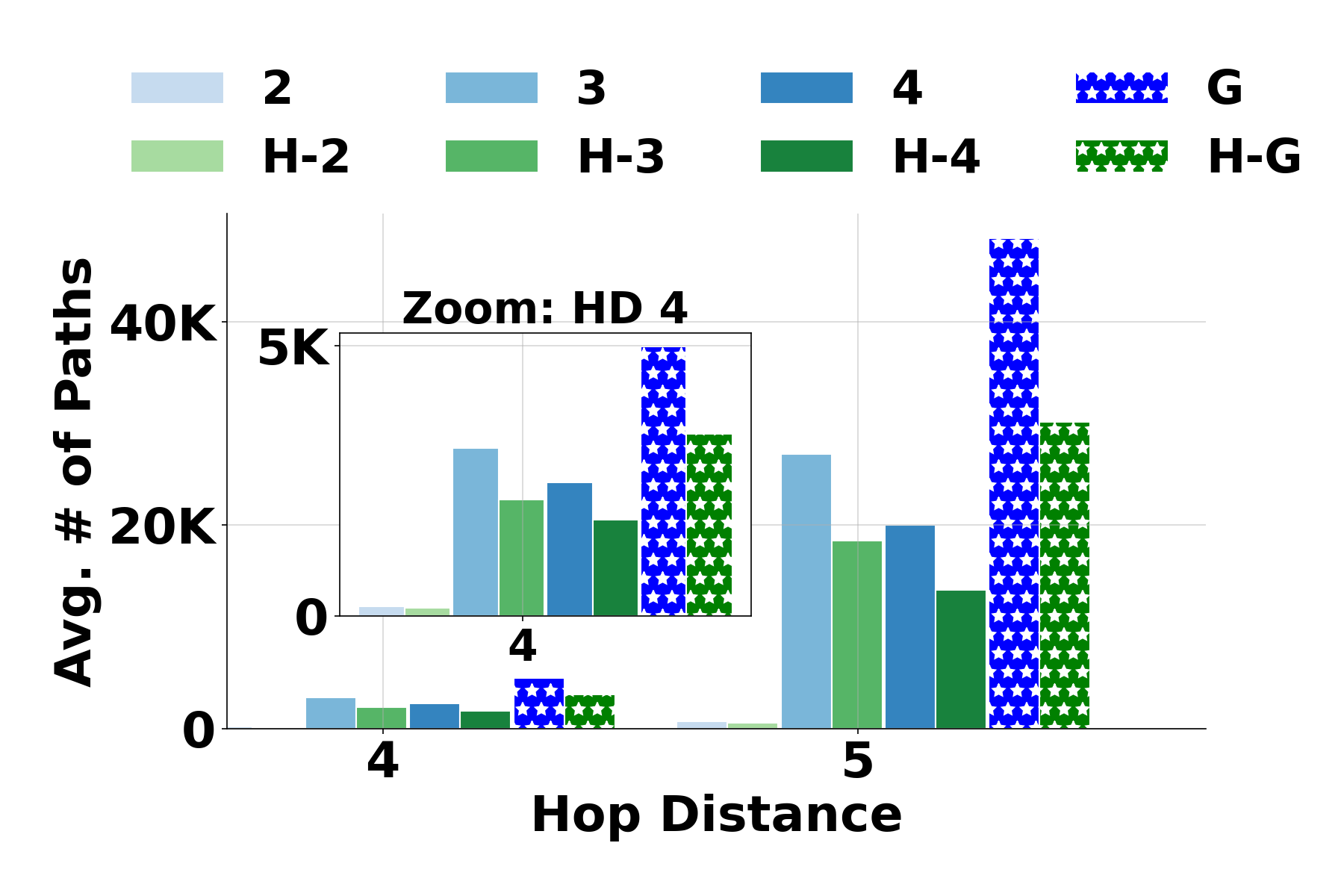}
        \caption{ANP Cora}
        \label{fig:cora-org-HD}
    \end{subfigure}%
    ~
    \begin{subfigure}{0.45\textwidth}
        \centering
        \includegraphics[width=\textwidth]{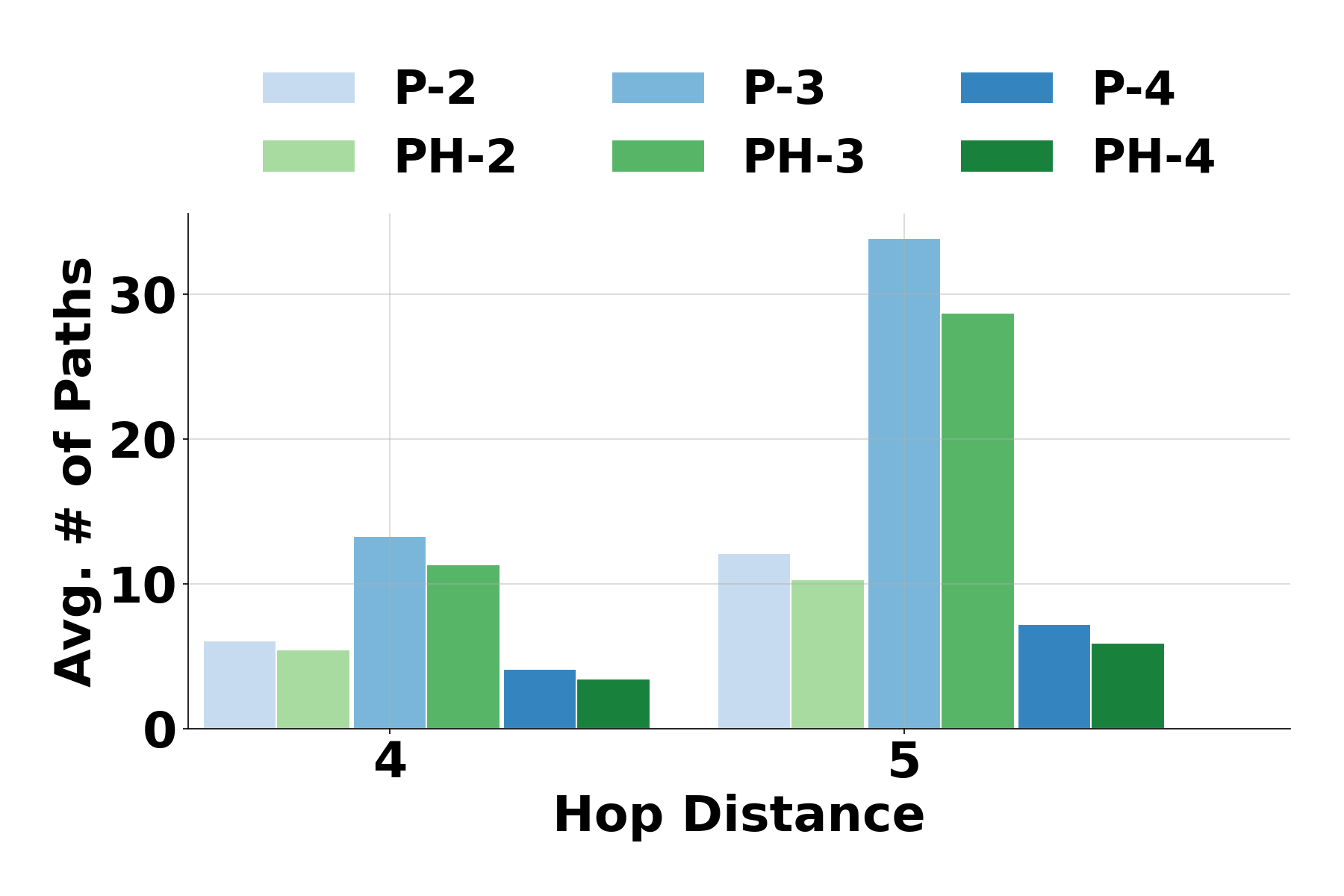}
        \caption{PS-ANP in Cora}
        \label{fig:cora-subg-HD}
    \end{subfigure}%
   
    \begin{subfigure}{0.45\textwidth}
        \centering
        \includegraphics[width=\textwidth]{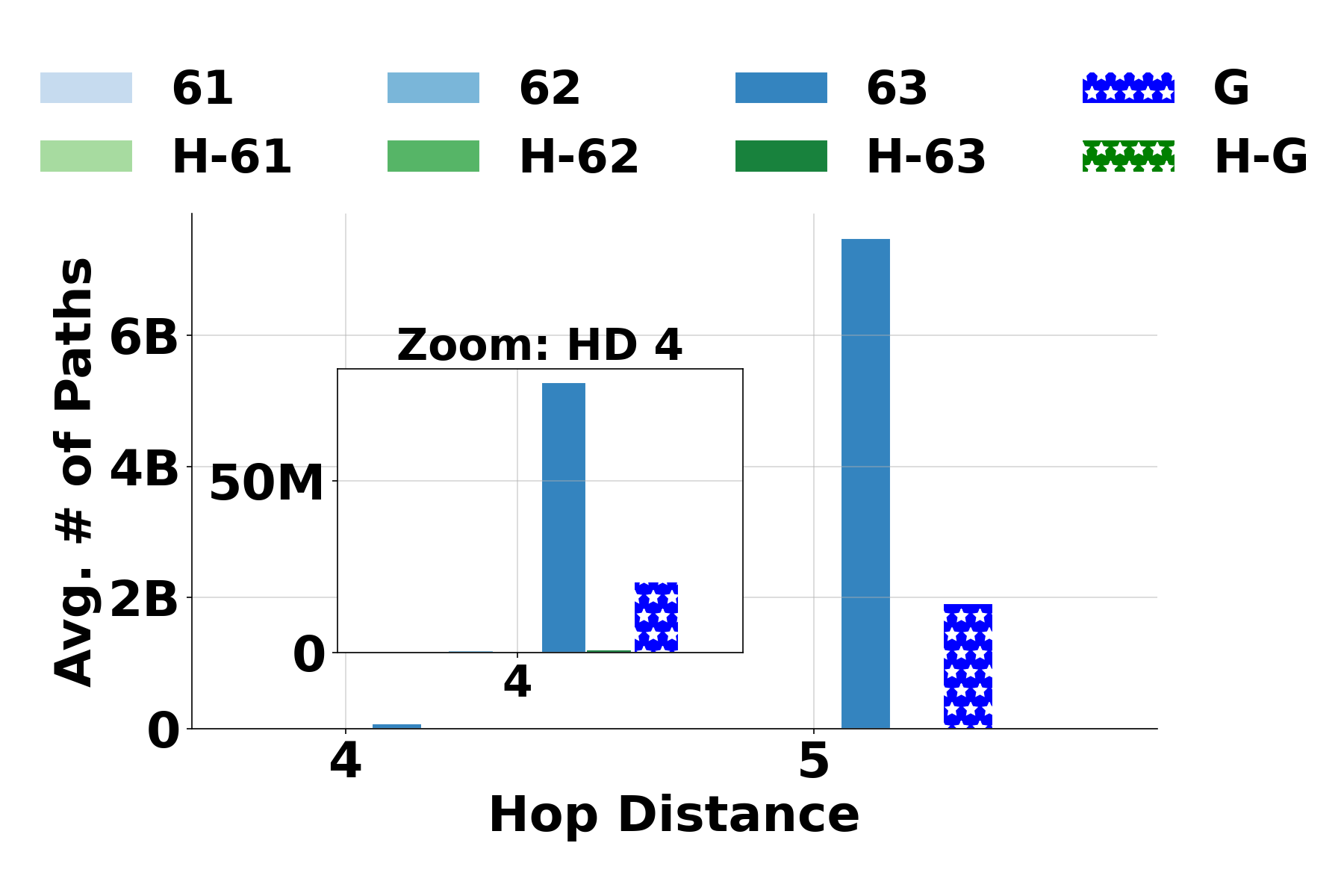}
        \caption{ANP in Chameleon}
        \label{fig:cham-org-HD}
    \end{subfigure}%
    ~
    \begin{subfigure}{0.45\textwidth}
        \centering
        \includegraphics[width=\textwidth]{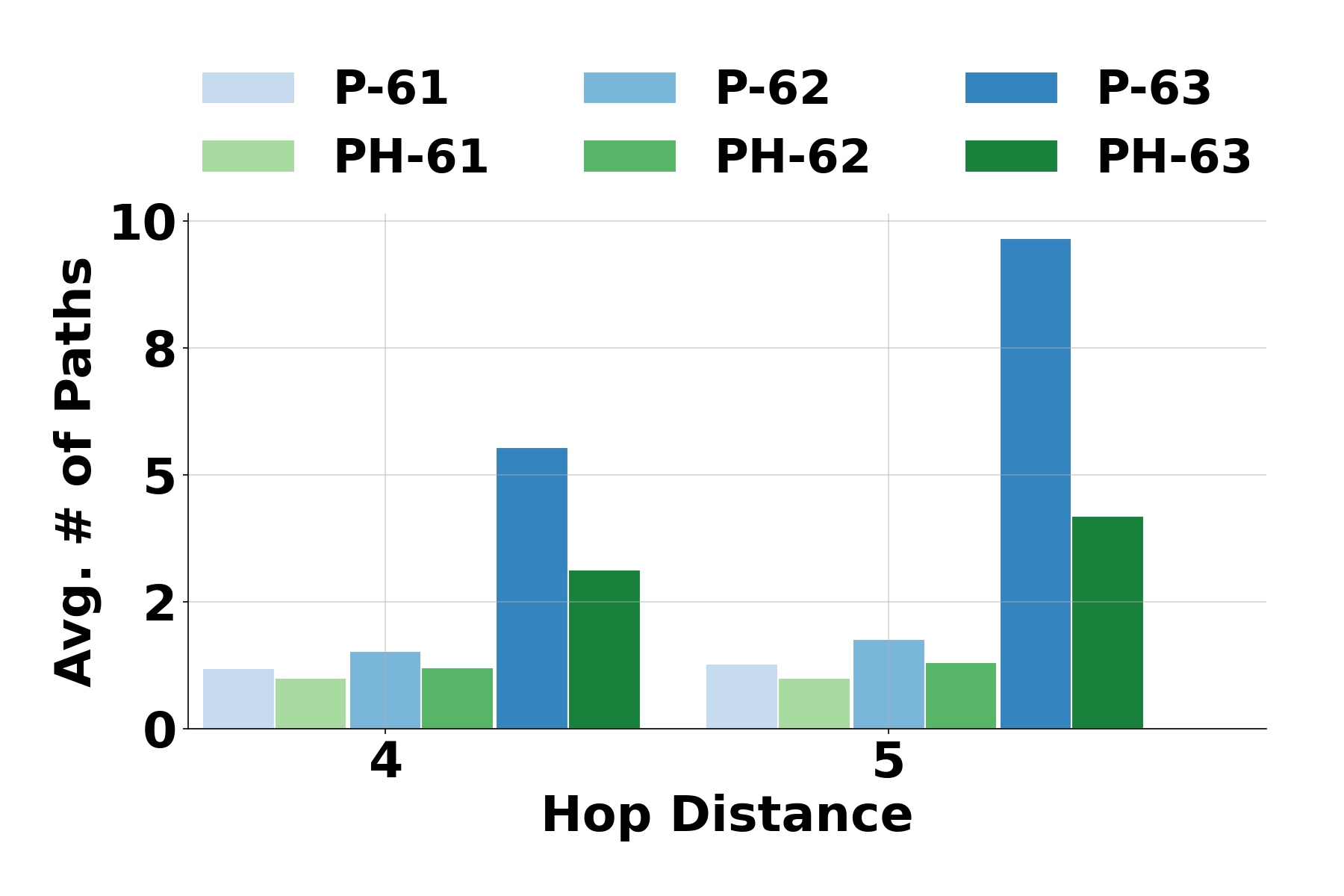}
        \caption{PS-ANP in Chameleon}
        \label{fig:cham-subg-HD}
    \end{subfigure}%
    
    \caption{\textbf{Pilot Study}. Average number $(\#)$ of paths (ANP) per node for $n \in \{4,5\}$ hop distances (HD) in Cora and Chameleon (Chm) datasets. Blue bars denote the original graph ($G$), its cores ($k$) and their (PS)-pooled subgraphs ($P-k$); green bars present their homophilic counterparts ($H-G$, $H-k$ and $PH-k$). From right to left, the deep (blue \& green) bar pairs (Figs.~\ref{fig:cora-org-HD} and~\ref{fig:cham-org-HD}) with hatch $('*')$ present the ANP of original graph and its homophilic subgraph respectively; in other cases, more deeper bar color denotes more denser subgraph. K - thousand and, M|B - m|billion.}
    \label{fig:oversquashing-paths}
    \vspace{-2mm}
\end{figure*}
\subsection{Pilot Study}
\label{susbsec:pilot_study}
Oversquashing occurs because GNNs struggle to learn long-range dependencies due to the bottleneck effects in graph structure~\cite{alon2020bottleneck}. Cohesion-aware graph decomposition not only reduces burdens on bottlenecked channels, but also preserves essential structural properties. However, in heterophilic networks, selecting task-relevant neighbors remains challenging due to cross-class mixing. This pilot study demonstrates that applying $k$-core decomposition, followed by SAGPooling, preserves essential class-consistency among vertices even in heterophilic graphs. That enables GNNs to aggregate more reliable signals, producing stable node representations. 

\begin{definition}(\textit{\textbf{Edge score}}):
In the context of \textbf{$k$-core}, for a graph $G = (V, E)$ and $k \ge 1$, an edge $(u, v) \in E$ can appear in multiple $k$-core subgraphs. The score of an edge is assigned as
\begin{align}
C(u, v) = max\{k~|~(u, v) \in G_{k}\}
\label{eq:3}
\end{align}
$C(u, v)$ is computed for the highest $k$-valued subgraph where $(u, v)$ exists. 
\label{def:2}
\end{definition}
Fig.~\ref{fig:oversquashing-paths} presents a pilot study on  Cora (homophilic) and Chameleon (heterophilic) datasets to examine how well subgraph decomposition preserves assortativity and disassortativity. According to definitions~\ref{def:1} and~\ref{def:2}, we first compute edge scores and extract edge-induced dense subgraphs for top-three $k-$core values. Form Cora $k \in \{2, 3, 4\}$ and from Chameleon $k \in \{61, 62, 63\}$. Next, we compute the average number of paths (ANP) for vertices at hop distances $n \in \{4, 5\}$ across original graphs, subgraphs and their homophilic counterparts $H = \{(u,v) \in E: y_u = y_v\}$. The measures are presented in bar graphs.

According to Figs.~\ref{fig:cora-org-HD} and~\ref{fig:cham-org-HD}, the subgraphs in both datasets retain the homophilic and heterophilic $(y_u \neq y_v)$ properties of their original graphs. While this preservation is often sufficient for homophilic networks to maintain task-relevant features, it becomes more challenging for heterophilic graphs like Chameleon. To extend the study, we apply SAGPool to each subgraph and repeat the ANP evaluation on pooled subgraphs (PS-ANP). Interestingly, in the Chameleon dataset (Fig.~\ref{fig:cham-subg-HD}), the pooled subgraphs exhibit a higher ratio of homophilic paths to the total average number of paths per node compared to the initial evaluation. That facilitates effective representation learning in graph models~\cite{gu2023homophily}.
Besides, it increases the ratio for the homophilic dataset. 

\subsection{Model Architecture} \label{subsec:model_architechture}
In Figure~\ref{fig:process_diagram}, the model architecture is comprised of four modules. \textcolor{gray8}{\textbf{Module (1)}} works as the graph pre-processor, while \textcolor{red}{\textbf{module (2)}} trains the decomposed subgraphs with GNNs. On the other hand, \textcolor{ForestGreen}{\textbf{module (3)}} functions as a feature post-processor, whereas \textcolor{deepyellow}{\textbf{module (4)}} operates for the final evaluation. 

\subsubsection{$(1)$ Cohesive Subgraph Decomposition.} Density-informed subgraphs, called cohesive subgraphs, are crucial for effective graph structure learning where each vertex gains sufficient connectivity in a particular region. Hence, learning the representation of vertices with these cohesive regions can capture adequate proximal insight for them. The $k$-core is one of the popular algorithms to obtain cohesion-focused subgraphs~\cite{chen2024breaking}. 

\begin{theorem} \label{def:3}
(\textit{\textbf{$CaEF:$ Closure-aware Edge Filtration}}):
Let $G_{k}$ denote the $k$-core of $G = (V, E)$.  For an edge $e =(u, v) \in G_{k}$, where $N(u)$ and $N(v)$ are the neighbor set of node $u$ and $v$, its triadic support is defined as   
\begin{equation}
S (u, v) = |N(u) \cap N(v)|
\label{eq:4}
\end{equation}
 For $k \ge \delta$, if $S(u, v) = 0$ then $(u, v)$ is removed from $G_{k}$ and reassigned to previous core $C (u, v) = (k-1)$, where $\delta$ denotes the edge filtering threshold. 
\end{theorem}

In this module, to partition a graph into cohesive subgraphs,  $\cacose$ employs the $k$-core decomposition algorithm in~\cite{batagelj2003m}. Initially, it calculates the edges coreness following the definition (\ref{def:2}). For example, edge $(v_4, v_6)$ is a part of both the $1$-core and the $2$-core subgraphs. As $(2>1)$, the coreness score $C(v_4, v_6) = 2$. Similarly, $(v_6, v_7) \in G_{1} \cap G_{2}$, hence, $C(v_6, v_7) = 2$. 

Simultaneously, we examine the existence of narrow edges in $k$-core subgraphs. According to theorem (\ref{def:3}), such edges are removed from $k$-core subgraph and reassigned to the $(k-1)$-core subgraph. For example, in Figure~\ref{fig:process_diagram} when $k=3$, the edge $(v_4, v_7)$ has no support. It may act as a noisy channel during neighborhood aggregation in GNNs. Hence, for $\delta=3$, $(v_4, v_7)$ is pruned from $G_{3}$ and assigned to $G_{2}$, then $C(v_4, v_7) = 2$. After assigning scores to each of the edges, the graph is decomposed into subgraphs corresponding to the same edge-score groups -
\begin{equation}
\begin{aligned}
S &= \{ S_{k} \}_{k=1}^{k_{\max}}, \quad S_{k} = (V_{S}, E_{S}), \\
E_{S} &= \{ (u,v) \in E \mid C(u,v) = k \}
\end{aligned}
\label{eq:5}
\end{equation}
For example, in Fig.~\ref{fig:process_diagram}, the final set of decomposed subgraphs is $S = \{S_2, S_3, S_4\}$. Note that $S_{k} \neq G_{k}$, e.g $S_{3} \neq G_{3}$. Although many nodes overlap across different subgraphs, they consistently carry the same weighted edges within each subgraph. Hence, these score-based partitions avoid structural inconsistencies and preserve meaningful subgraphs for graph operations. 
\begin{figure*}[t!]
    \centering
    \includegraphics[width=\textwidth]{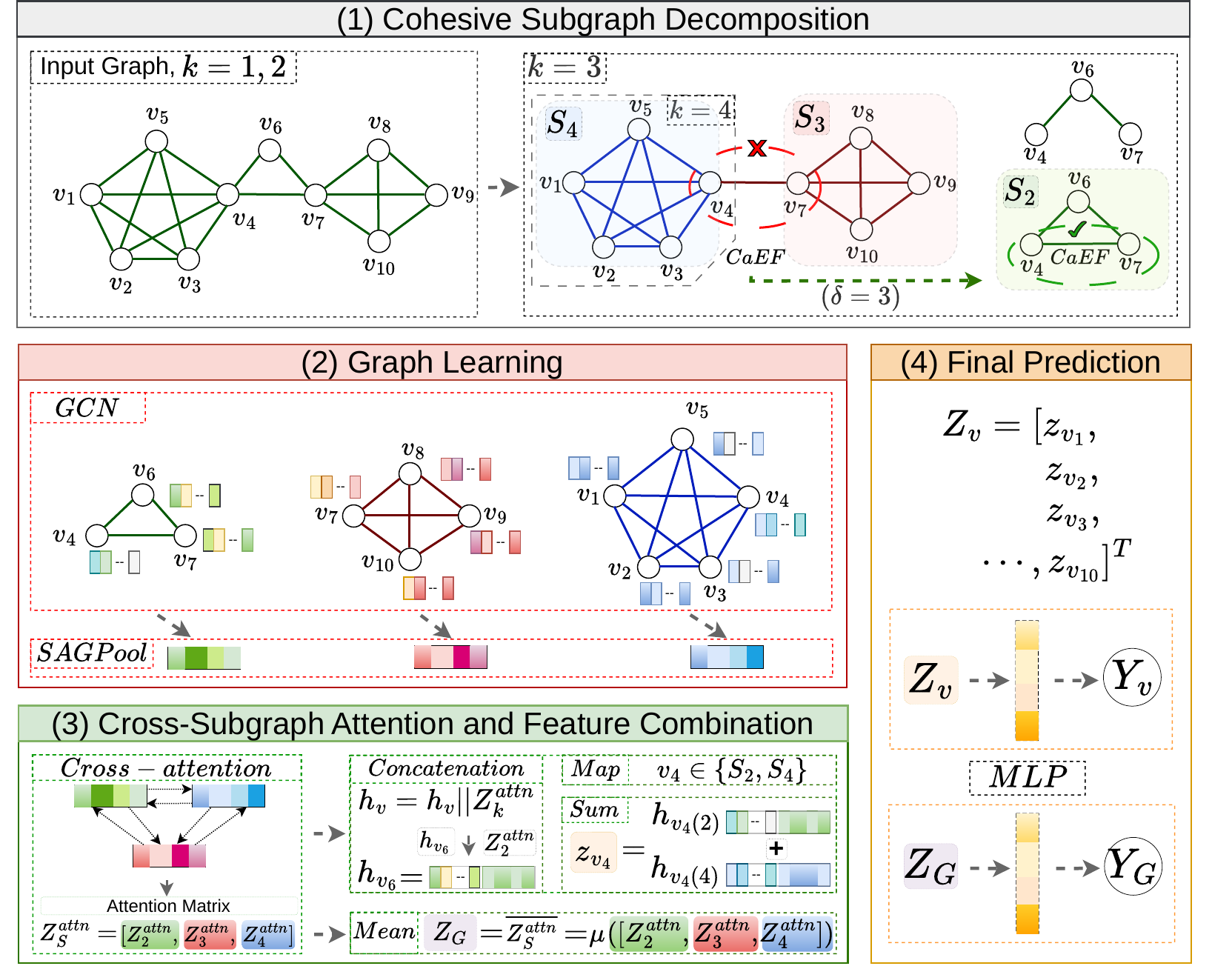}
    \caption{ Architecture of \textbf{$\cacose~(\delta = 3)$.}} 
    
    
    \label{fig:process_diagram}
\end{figure*}

\subsubsection{$(2)$ Graph Learning.} In this part, $\cacose$, particularly aims to embed the topological information of all decomposed subgraphs into node representations. For this purpose, it applies the $GCN$ to each subgraph and obtains the representation of each vertex $h_v \in H_{k}$ separately. 
\begin{align}
H_{k} = \text{GCN}_{k}(S_{k}) 
\label{eq:6}
\end{align}
$H_{k} \in \mathbb{R}^{N_{v_S}\times h}$ presents the node feature matrix, where $N_{v_S} = |S_{k}|$ denotes the number of vertices in the subgraph. Nevertheless, Just learning within subgraph get local cohesive information but loss global information. To get the global information of each subgraph, $\cacose$ attains subgraph embeddings via employing the self-attention graph pooling $(SAGPool)$. 
\begin{align}
Z_{k} = \text{SAGPool}_{k}(S_{k},H_{k})
\label{eq:7}
\end{align}
This pooling encodes essential global structural information of the subgraphs through relevant neighborhood selection, $Z_{k} \in \mathbb{R}^{d_{S}}$, where, $d_{S}$ presents the pooling dimension. It effectively filters out candidate nodes' task-irrelevant or noisy neighbors even from highly heterophilic networks.  Combining this subgraph embeddings with node embeddings provides crucial global information to vertices. 

\subsubsection{$(3)$ Cross-Subgraph Attention \& Feature Combination.} 
Cohesion-centric partitioning reduces bottleneck loads but causes loss of long-range dependency among vertices. Additionally, the subgraphs' embeddings via pooling capture only local structural information while overlooking other subgraphs. Hence, a way to recover global connectivity across partitions is required. The attention mechanism~\cite{vaswani2017attention} enables the modeling of long-range dependencies through sequential learning. Toward this goal, cross-subgraph attention facilitates communication and information propagation among distinct sequence subgraphs. 

In this stage, $\cacose$ employs cohesion-sensitive subgraph attention to update the subgraph embeddings. Each subgraph embedding is processed through cross-attention to encode mutual information across regions. Attention~\cite{vaswani2017attention} help in capturing essential awareness among entities. The pooled subgraphs’ embeddings are represented as, $Z_{S} =\begin{bmatrix} Z_{1},Z_{2}\dots Z_{k_{max}}\end{bmatrix}^\top \in \mathbb{R}^{N_{S}\times d_S}$, where $N_{S} = |S|$ denotes the number of subgraphs and  $d_S$ as feature dimension. From the subgraphs feature matrix, the query, key, and value matrices are computed as $Q = Z_{S}W_Q,~ K= Z_{S}W_K,~ V = Z_{S}W_V$. Here, $W_Q, W_K, W_V$ are learnable weight matrices. Then, the attention scores are measured and subgraphs representations are updated as:
\begin{equation}
Attn_{S} = \mathrm{softmax}\!\left(\frac{QK^T}{\sqrt{d_S}}\right),
\quad
Z^{attn}_{S} = Attn_{S} V
\label{eq:8}
\end{equation}

$Z^{attn}_{S} = \begin{bmatrix} Z_{1}^{attn},Z_{2}^{attn}\cdots Z_{k_{max}}^{attn}\end{bmatrix}^T \in \mathbb{R}^{N_{S} \times d_S}$ presents the updated representation of a subgraph, incorporating other subgraphs' attentions where $Attn_{S} \in \mathbb{R}^{N_{S} \times N_{S}}$ denotes subgraphs' attention matrix. Such as, in Fig~\ref{fig:process_diagram}, $\cacose$ computed the cross-attentive matrix as,  $Z^{attn}_{S} = \begin{bmatrix} Z_{2}^{attn},Z_{3}^{attn}, Z_{4}^{attn}\end{bmatrix}^T$. 
It is noteworthy that the procedure of the attention mechanism seems like complete graph learning. Meanwhile, cohesion-sensitive partitions produce a smaller number of subgraphs; therefore, cross-subgraph attention adds only negligible computational overhead. The final graph representation is derived by taking the mean $(\mu)$ of the cross-attentive subgraphs embeddings, preserving inter-subgraph relational information.  
\begin{flalign}
Z_G = \overline{Z_{S}^{attn}} = \mu(\begin{bmatrix} Z_{1}^{attn},Z_{2}^{attn}\cdots Z_{k_{max}}^{attn}\end{bmatrix})
\label{eq:9}
\end{flalign}
$\cacose$ obtains informative nodes' embeddings by concatenating $(\parallel)$ each subgraph's attentive features with its vertices representations.
\begin{flalign}
h_v = h_v \parallel Z^{attn}_{k},\quad v \in S_{k}
\label{eq:10}
\end{flalign}
 If a node belongs to multiple subgraphs, a mapping $(Map)$ function matches its occurrences, and all of its representations are summed. Such as, for $v \in \{S_{l}, S_{m},  S_{n}\}$ and  $S_{l}, S_{m}, S_{n} \subseteq S$, the final node representations is computed as:
 \begin{flalign}
z_v = \sum_{ k \in \{l,m,n\}}h_{v(k)}
\label{eq:11}
\end{flalign}
For instance, in Fig~\ref{fig:process_diagram}, node $v_4$ is part of both $S_2$ and $S_4$. Hence, the final representations is summed $(Sum)$ as: $z_{v_4} = h_{v_4(2)} + h_{v_4(4)}$. 
\subsubsection{$(4)$ Final Prediction.}
Decomposition reduces excessive information processing through bottlenecked channels in graph. Subsequently, selection-based learning provides compact and meaningful subgraph representations. Following that, cross-subgraph attention alleviates long-hop dependency by modeling interactions among vertices across different subgraphs. Finally, the processed nodes' $(Z_{v} = \begin{bmatrix} z_{v_1},z_{v_2}\cdots z_{v_{N_v}}\end{bmatrix}^T)$ and graphs' $(Z_{G})$ representations are passed through a multilayer perceptron $(MLP)$ to produce final predictions $Y_v$ (NC) and $Y_G$ (GC). Experimental results validate the model's ability in mitigating oversquashing and enhance performance on downstream inference. The detailed examination of the $\cacose$ algorithm, its complexity, theoretical justification, and scalability analysis are available in Appendix~\ref{sec:appendix}. 
\section{Related Works}\label{sec:related_works}
\textbf{Graph Decompositions.} Numerous graph decomposition algorithms are applied to improve GNNs' effectiveness. At the earlier stage of GNNs, spectral clustering~\cite{zhang2021spectral,bianchi2020spectral} methods were much more popular along with the hierarchical~\cite{ying2018hierarchical} and modularity-based~\cite{tsitsulin2023graph} clustering models. However, due to expensive eigenvalue decomposition and clustering assignment steps, these methods encounter relatively higher time complexity.



\noindent\textbf{Cohesion-sensitive Decompositions.}
These algorithms are mostly applied to determine the interconnectedness between nodes in multiple network regions for solving different domain problems: graph compression~\cite{akbas2017truss}, high-performance computing~\cite{liu2024parallel}, analyzing social networks~\cite{chen2024breaking}, etc. Only a few methods utilize these algorithms to enhance the effectiveness of GNNs. TGS~\cite{hossain2024tackling} utilizes the $k$-truss algorithm for edge scoring and sparsifies noisy edges to mitigate oversmoothing in GNNs. Another model, CTAug~\cite{wu2024graph}, utilizes $k$-truss and $k$-core algorithms to provide cohesive subgraph awareness and improve graph contrastive learning (GCL). Nonetheless, due to repetitive edge filtering for sparsification and GCL’s resource-expensive nature, both models require longer runtime.  

\noindent\textbf{Oversquashing.} Prior work~\cite{alon2020bottleneck} illustrates that due to bottlenecks in the graph, long-range neighborhood signals are distorted, which downgrades GNNs' performance in downstream graph learning tasks. SDRF~\cite{topping2021understanding} proposed a curvature-based graph rewriting to detect edges responsible for the information bottleneck. FoSR~\cite{karhadkar2022fosr} applies systematic edge addition operations in graphs in accordance with spectral expansion, while BORF~\cite{nguyen2023revisiting} demonstrates the impacts of curvature signs for oversquashing and oversmoothing. However, due to dependency on intermediate layer output and subgraph matching with optimal transport, these models show inconsistency in large-scale graphs' operations. Another approach, GTR~\cite{black2023understanding}, uses effective resistance and a repeated edge addition technique to reduce the impact of oversquashing in GNN. However, this method focuses on the entire graph's information for edge rewiring, which increases complexity in executing larger graphs. Recently, GOKU \cite{liang2025mitigating} utilizes spectrum-preserving graph rewiring, first densifying the network and subsequently applying structure-aware graph sparsification. In contrast, LASER~\cite{barbero2023locality} adopts locality aware sequential rewiring while considering multiple graph snapshots. GraphViT~\cite{he2023generalization} leverages the Vision Transformers combined with MLP mixing mechanism to model long-range dependencies in graphs. On the other hand, LRGB~\cite{dwivedi2022long} addresses the limitations of  benchmark GNNs and Transformers in analyzing  long-range vertex interactions. However, both approaches primarily focus on relatively ordered and structurally regular datasets. In contrast, our work concentrates on reducing oversquashing in complex, irregular real-world and social networks. 

\section{Experiment Results}\label{sec:experiment}
This section covers $\cacose$'s experiment outcomes' information. First, it discusses the datasets and baselines. Then, it describes the experimental setup and, finally, presents the model's results compared to other baselines with various experiments. \textbf{Note that, Prior works~\cite{dwivedi2022long,alon2020bottleneck} experiment on synthetic datasets (tree neighbors-match, ring transfer, etc.) are unsuitable for $k$-core due to their fixed structure. Hence, our experiment mostly focuses on real-world complex networks.} 
\subsection{Experimental Settings}
\subsubsection{Datasets and Baselines.} Our $\cacose$ model is evaluated on $8$ datasets $(D_v)$ for the node classification task. Where five datasets are homophilic:  \textbf{Cora}, \textbf{C}ite\textbf{Seer}, \textbf{Co}Author \textbf{CS}, \textbf{Am}azon \textbf{C}o\textbf{mp}uters and \textbf{Am}azon \textbf{P}ho\textbf{to}s. Besides, three datasets are heterophilic: \textbf{Cham}eleon, \textbf{Squir}rel and \textbf{Texas}. In case of graph classification tasks we experiment on six different datasets where four from social network domain: \textbf{IMDB-B}INARY, \textbf{IMDB-M}ULTI, \textbf{COLLAB} and \textbf{R}ED\textbf{D}I\textbf{T-B}INARY. Other two from the biomedical domian: \textbf{MUTAG} and \textbf{PROT}EINS.  

Our model is compared with the $11$ standard methods for the node classification \textbf{(NC)} tasks: GCN~\cite{kipf2016semi}, GAT \cite{velivckovic2017graph}, SAGE \cite{hamilton2017inductive}, LRGB \cite{dwivedi2022long}, 
BORF \cite{nguyen2023revisiting}, FOSR \cite{karhadkar2022fosr}, SDRF \cite{topping2021understanding}, GTR \cite{black2023understanding}, DR \cite{attali2024delaunay}, LASER \cite{barbero2023locality} and GOKU \cite{liang2025mitigating}. For graph classification \textbf{(GC)}, we compare $\cacose$ with seven methods (BORF, FOSR, SDRF, GTR, DR, LASER and GOKU), excluding the first four baselines. 

\begin{table}[t!]
\centering
\caption{Node Classification Accuracy Comparison. Highest accuracy is \textbf{bolded}, second-best is \underline{underlined}. OOM: Out of Memory in execution.}
\label{tab:res_node_classification}
\begin{tabular}{@{}lcccccccccc@{}}
\hline
\cellcolor{gray!15} \textbf{M / $D_v$} & \cellcolor{gray!15} Cora & \cellcolor{gray!15} C.Seer & \cellcolor{gray!15} Co. CS & \cellcolor{gray!15} Am. Cmp & \cellcolor{gray!15} Am. Pto & 
\cellcolor{gray!15} Texas &
\cellcolor{gray!15} Cham. & \cellcolor{gray!15} Squir. \\ \hline
GCN & \underline{84.24} & 69.10 & 88.86 & 87.77 & 92.12  & 52.36 & 63.32 & \underline{48.79} \\
GAT & \textbf{85.00} & 67.94  & 87.13 & 88.41 & 92.01 & 52.37  & 61.70 & 46.18 \\
SAGE & 83.60 & 67.76  & 88.42 & 87.22 & 91.58 & 56.05  & 62.64 & 47.15 \\
LRGB & 71.02 & 57.58 &  59.50 & 70.34 & 74.04 & \underline{57.36} & 46.93 & 30.49 \\
BORF & 83.68 & 67.08 & \underline{90.52} & 89.79 & 91.93 & 54.21 & 60.24 & OOM \\
FOSR & 83.66 & 67.24 & 90.44 & \underline{89.82} & 91.83 & 52.37 & 60.31 & 40.19 \\
SDRF & 84.04 & 67.42  & 90.56 & 86.79 & 91.71 & 52.63 & 60.74 & 43.02 \\
GTR & 84.07 & \underline{69.15}  & 88.87 & 85.37 & 91.76 & 52.37 & 63.78 & 48.67 \\
DR & 44.35 & 24.90  & 72.50 & 71.43 & 79.86 & \textbf{69.45} & 27.87 & 23.85 \\
LASER & 75.77 & 64.26 & 76.90 & 38.63 & 67.36 & 32.63  & 41.90 & 25.87 \\
GOKU & 82.66 & 66.08 & 90.04 & OOM & \underline{92.51} & 37.10 & \underline{65.90} & 46.74 \\ \hline
\textbf{$\cacose$} & \textbf{85.00} & \textbf{69.42} &  \textbf{90.73} & \textbf{90.43} & \textbf{92.98} & 
54.47 &
\textbf{68.99} & \textbf{58.86} \\ \hline
\end{tabular}
\end{table}

\begin{table}[t!]
\vspace{-3mm}
\centering
\caption{Comparison of different methods on graph classification benchmarks.}
\label{tab:res_graph_classification}
\begin{tabular}{lcccccc}
\hline
\cellcolor{gray!15} M/ $D_G$ & \cellcolor{gray!15} MUTAG & \cellcolor{gray!15} IMDB-B & \cellcolor{gray!15} \cellcolor{gray!15} IMDB-M & \cellcolor{gray!15} RDT-B & \cellcolor{gray!15} COLLAB & \cellcolor{gray!15} PROTEINS \\
\hline
SDRF & 74.53 & 62.90 & 41.53 & 85.40 & 70.22 & 66.88 \\
FOSR & 75.89 & 60.40 & 37.33 & 83.25 & 69.85 & 66.70 \\
BORF & 64.00 & 60.82 & 38.20 & 84.92 & OOM & 68.41 \\
GTR & \underline{76.00} & \underline{70.20} & 45.33 & \textbf{89.65} & 68.02 & 71.52 \\
DR & 71.00 & 53.60 & 35.60 & 73.60 & 55.54 & \textbf{72.23} \\
LASER & 68.00 & 67.30 & \underline{45.40} & 81.70 & \underline{71.32} & 71.07 \\
GOKU & 74.50 & 64.80 & 42.40 & 80.55 & 69.02 & 70.80 \\ \hline
\textbf{$\cacose$} & \textbf{76.99} & \textbf{73.20} & \textbf{49.14} & \underline{85.70} & \textbf{80.95} & \underline{71.79} \\ \hline
\end{tabular}
\end{table}
\vspace{-3mm}

\subsubsection{Implementation Details.} We evaluate our model by running baselines' codes for a fair comparison. Except for Cora $(1208, 500, 500)$ and CiteSeer $(1812, 500,$ $ 500)$, we split other datasets into $48\%$, $32\%$, and $20\%$ for training, validation, and testing, respectively. In the case of graph classification (GC) for all datasets, the ratio is $(80\%:10\%:10\%)$.  
Each model runs for up to $250$ iterations (NC) and $100$ iterations (GC), with early stopping after $50$ and $25$ consecutive epochs without validation improvement, respectively. Nodes' features are initialized with $1-hot$ encoding and hidden dimension in GNNs is set as $128$. For a learning rate of $2.5e-3$, $\cacose$ utilizes $l_2$ regularization with a weight decay of $1e-4$. The pooling ratio is set as $0.5$ in both cases while the number of heads are set as $2$ (NC) and $1$ (GC). Besides, we set the $CaEF$ threshold $\delta$ as $3$ for both tasks. Finally, we split each dataset using $10$ different seeds and report the mean accuracy as the final result. 

\subsection{Result Analysis}
\subsubsection{Node Classification and Graph Classification.} Table~\ref{tab:res_node_classification} demonstrates the comparison of our model with the baselines in the accuracy metric. In most cases, $\cacose$ achieves a superior gain $\left( \frac{acc(\cacose) - acc(baseline)}{acc(baseline)}*100\right)$ in ($\%$)
over other methods. Particularly on the dense heterophilic: Chameleon and Squirrel datasets it shows performance gains of $4.69\%$ and $20.63\%$, respectively, over the nearest performing baselines GOKU and GCN. 

Table~\ref{tab:res_graph_classification} presents the performance of our model in comparison to standard oversquashing addressing baselines. The $\cacose$ achieves better or almost similar performance over the baselines. Notably, on the IMDB-BINARY and COLLAB datasets, it surpasses the other baselines with substantial gains of $4.27\%$ and $13.50\%$, respectively, where the closest scoring baselines are GTR and LASER.  

\begin{figure*}[t!]
    \centering
    \begin{subfigure}{0.24\textwidth}
        \centering
        \includegraphics[width=\textwidth]{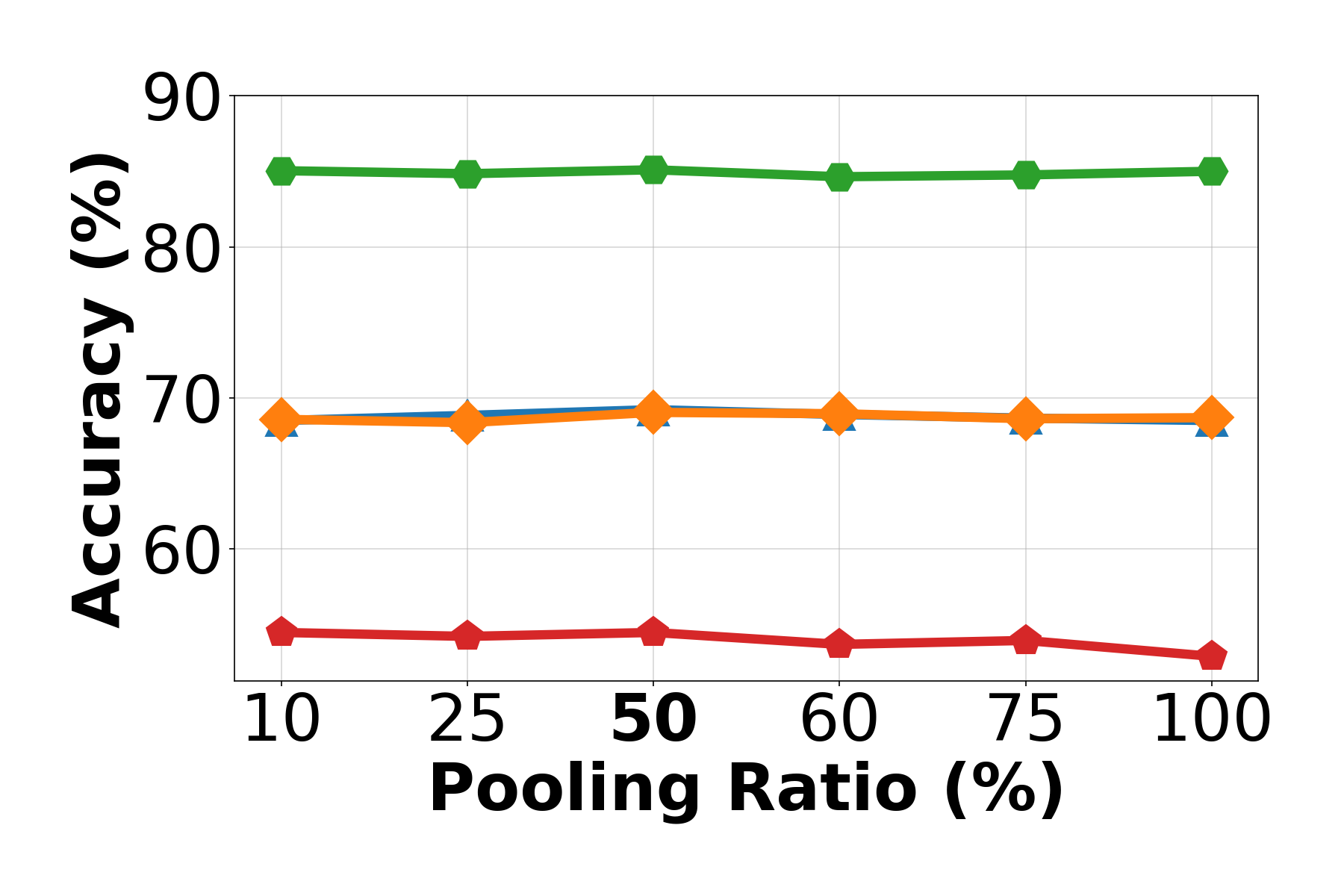}
        \caption{NC (PR)}
        \label{fig:pooling-ratio-k-core}
    \end{subfigure}%
    ~
    \begin{subfigure}{0.24\textwidth}
        \centering
        \includegraphics[width=\textwidth]{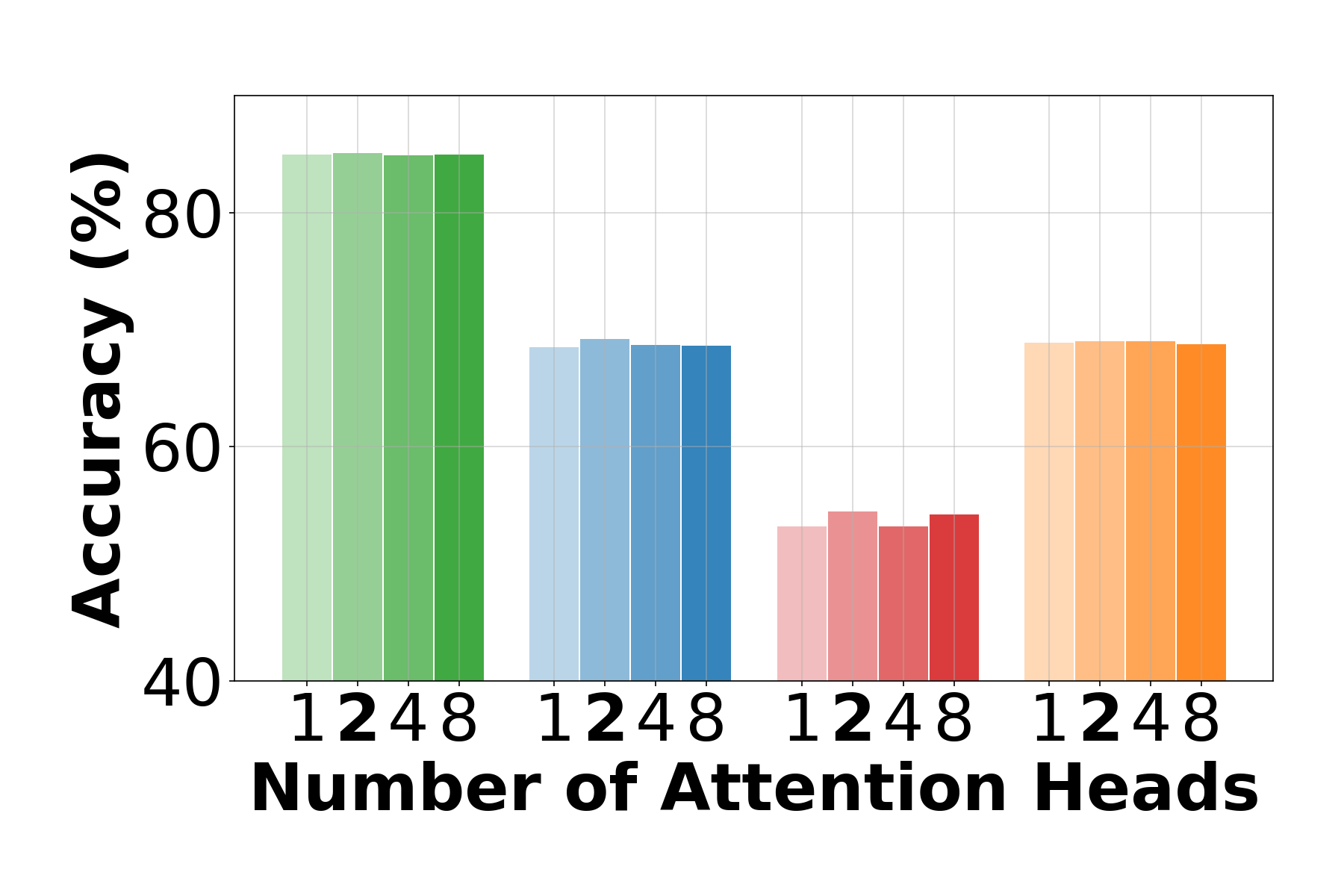}
        \caption{NC (NH)}
        \label{fig:num-heads-k-core}
    \end{subfigure}%
   ~
    \begin{subfigure}{0.24\textwidth}
        \centering
        \includegraphics[width=\textwidth]{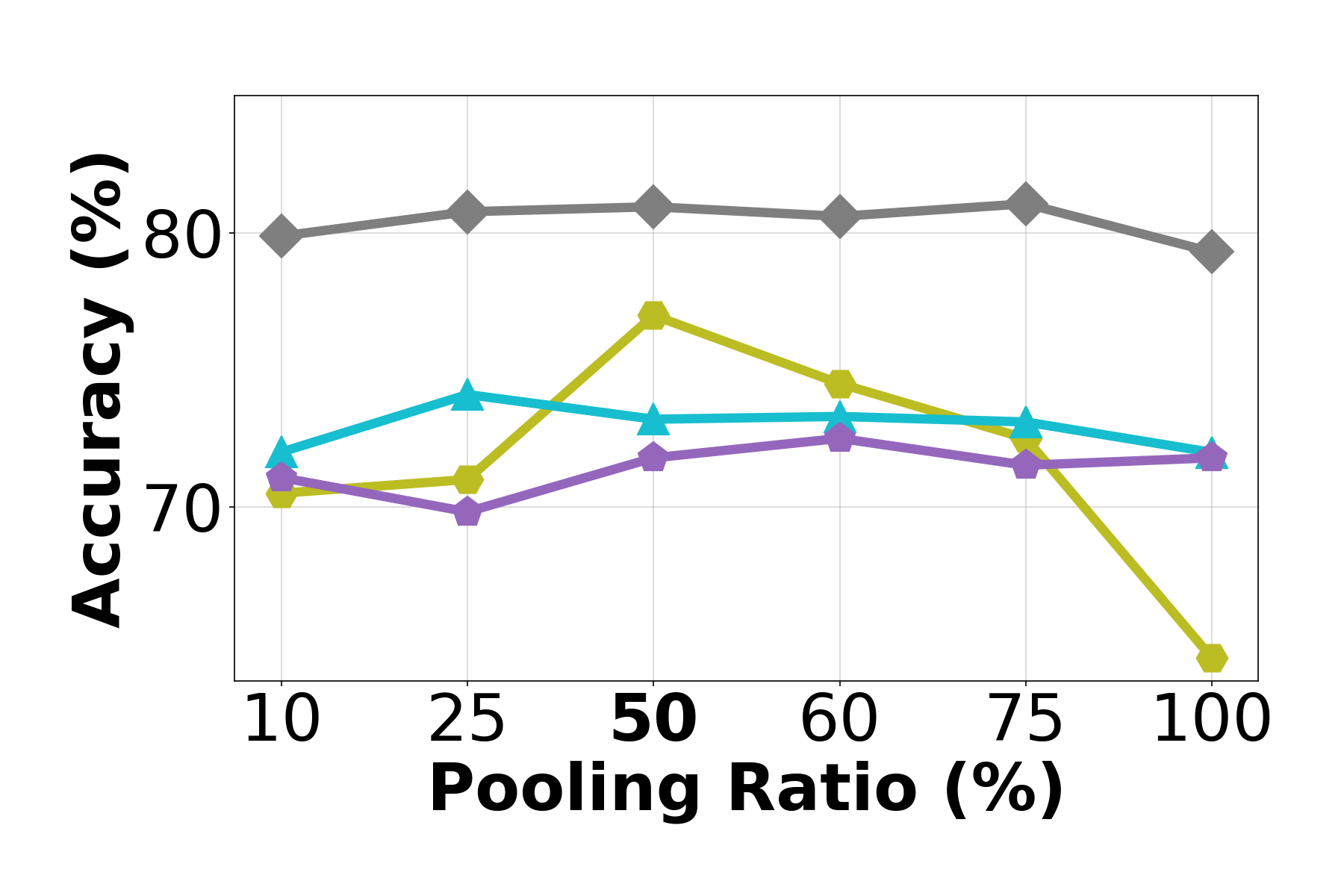}
        \caption{GC (PR)}
        \label{fig:pooling-ratio-gc-core}
    \end{subfigure}%
    ~
    \begin{subfigure}{0.235\textwidth}
        \centering
        \includegraphics[width=\textwidth]{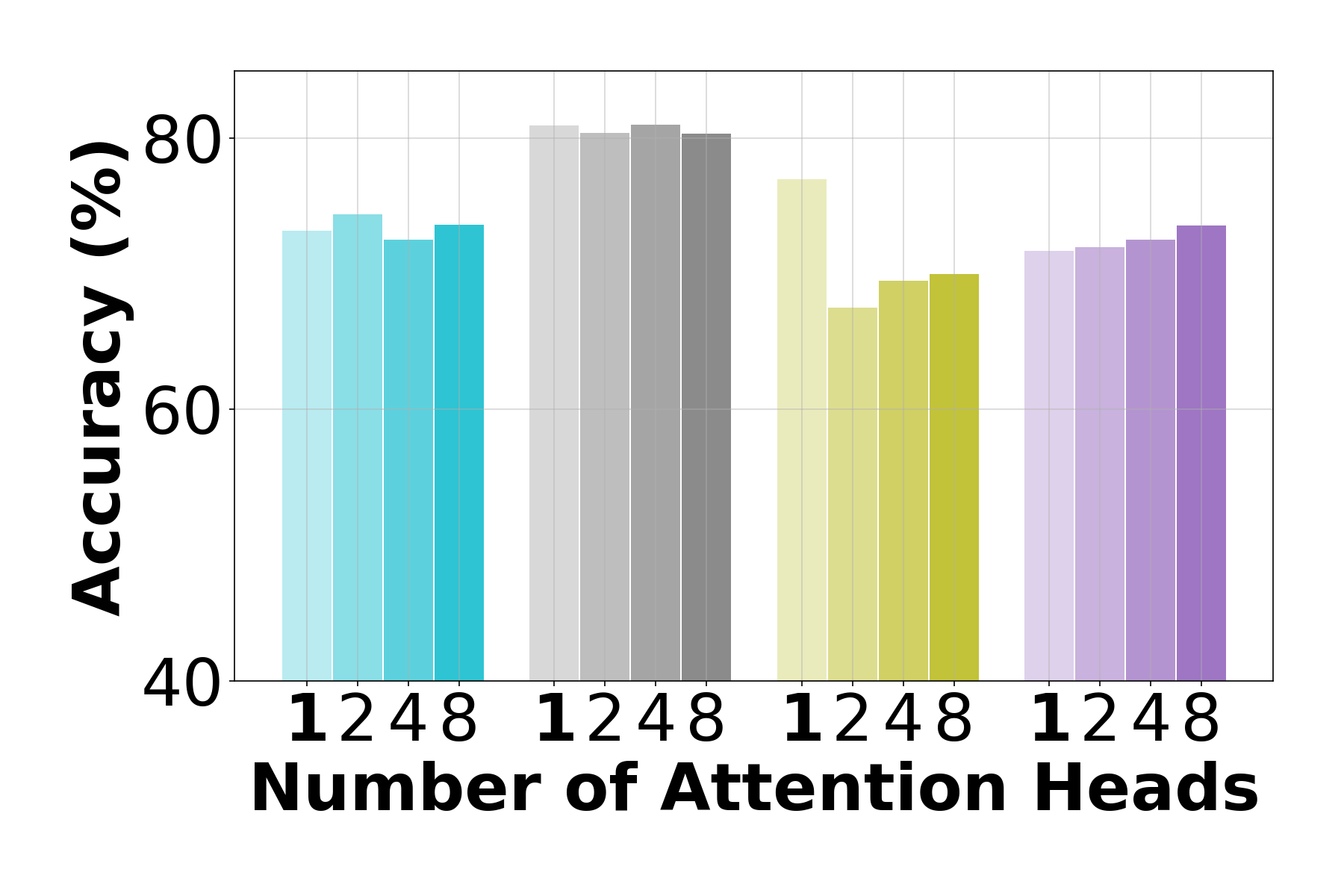}
        \caption{GC (NH)}
        \label{fig:num-heads-gc-core}
    \end{subfigure}%
    
    \caption{\textbf{Sensitivity Analysis}. Varying Pooling Ratios (PR) and Numbers of Heads (NH). $\cacose$'s settings for Node Classification (PR = 50\% and NH = 2) and for Graph Classification (PR = 50\% and NH = 1) are highlighted in bold. For NC datasets, \textcolor{ForestGreen}{\rule{0.6em}{0.6em}} Cora \;
\textcolor{tabblue}{\rule{0.6em}{0.6em}} CiteSeer \;
\textcolor{tabred}{\rule{0.6em}{0.6em}} Texas \; \textcolor{orange}{\rule{0.6em}{0.6em}} Chameleon. For GC datasets \textcolor{cyan8}{\rule{0.6em}{0.6em}} IMDB-B \;
\textcolor{gray8}{\rule{0.6em}{0.6em}} COLLAB \;
\textcolor{olive8}{\rule{0.6em}{0.6em}} MUTAG \; \textcolor{purple8}{\rule{0.6em}{0.6em}} PROTEINS.}
    \label{fig:sensitivity-analysis}
\end{figure*}
\vspace{-4mm}
\begin{table}[t!]
\centering
\caption{Performance comparison across different values of $\delta$. In $\cacose$, $\delta =3$. Bold fonts indicate better accuracy than $\cacose$. }
\begin{tabular}{l|c|c|c|c|c||l|c|c|c|c|c}
\hline
\cellcolor{gray!15} $D_v$ / \textbf{$\delta$} &\cellcolor{gray!15} \textbf{3} &\cellcolor{gray!15} 4 & \cellcolor{gray!15} 5 & \cellcolor{gray!15} 6 & \cellcolor{gray!15} 7 & \cellcolor{gray!15} $D_G$ / \textbf{$\delta$} & \cellcolor{gray!15} \textbf{3} & \cellcolor{gray!15} 4 & \cellcolor{gray!15} 5 & \cellcolor{gray!15} 6 & \cellcolor{gray!15} 7 \\ \hline
Cora & 85.00 & 84.88 & \textbf{85.22} & 84.96 & 85.10 & RDT-B & 85.70 & 83.50 & 83.20 & 83.15 & 81.80 \\
C.Seer & 69.42 & 69.12 & 68.90 & 68.96 & 69.00 & IMDB-B & 73.20 & 71.60 & \textbf{73.40} & 72.70 & \textbf{73.79} \\
Cham. & 68.99 & 68.73 & 68.46 & 68.58 & 68.93 & PROT. & 71.79 & \textbf{72.77} & \textbf{72.05} & 71.52 & \textbf{72.41} \\ \hline
\end{tabular}
\label{tab:performance_delta}
\end{table}

\subsubsection{Sensitivity Analysis.} 
We analyze the impact of the pooling ratio and number of heads on model performance. As shown in Fig.~\ref{fig:pooling-ratio-k-core}, the accuracy curve for $\cacose$'s exhibits minimal fluctuation on four (NC) datasets for lower to higher pooling ratios. In contrast on GC datasets, in Fig.~\ref{fig:pooling-ratio-gc-core} MUTAG's curve displays an uneven trend, while other three demonstrate a moderate movement.  Regarding the number of heads, Fig.~\ref{fig:num-heads-k-core} illustrates that, our model experiences small-scale performance changes on the Texas and CiteSeer datasets with nearly identical on other two NC datasets. For the GC datasets (in Fig.~\ref{fig:num-heads-gc-core}), there is a little fluctuation in the performance on IMDB-BINARY and COLLAB while PROTEINS exhibits a gradual increase in performance. 

Furthermore, in Table~\ref{tab:performance_delta}, we analyze $\cacose$'s performance by changing the $CaEF$ threshold $\delta$ across three NC datasets and three GC datasets. Generally, models performance is observed to decrease as the value of $\delta$ increase. However, fluctuates on some datasets, specially on IMDB-B and PROTEINS.  
\vspace{-4mm}
\subsubsection{Analysis on Bridges in Heterophilic Networks.} In experiments, $\cacose$ shows outstanding performance on heterophilic networks. To determine the underlying reason, we analyze the bridge edges that act as the narrow channels in graphs. For each bridge edge, we examine the 2-hop neighborhood of its endpoint nodes and color the vertices by class label. The same analysis is repeated on the corresponding edges in the decomposed subgraphs. In many cases, the neighborhoods around decomposed heterophilic bridges reveal latent homophily. 
\begin{figure*}[h!]
    \vspace{-6mm}
    \centering
    \begin{subfigure}{0.242\textwidth}
        \centering
        \includegraphics[width=\textwidth]{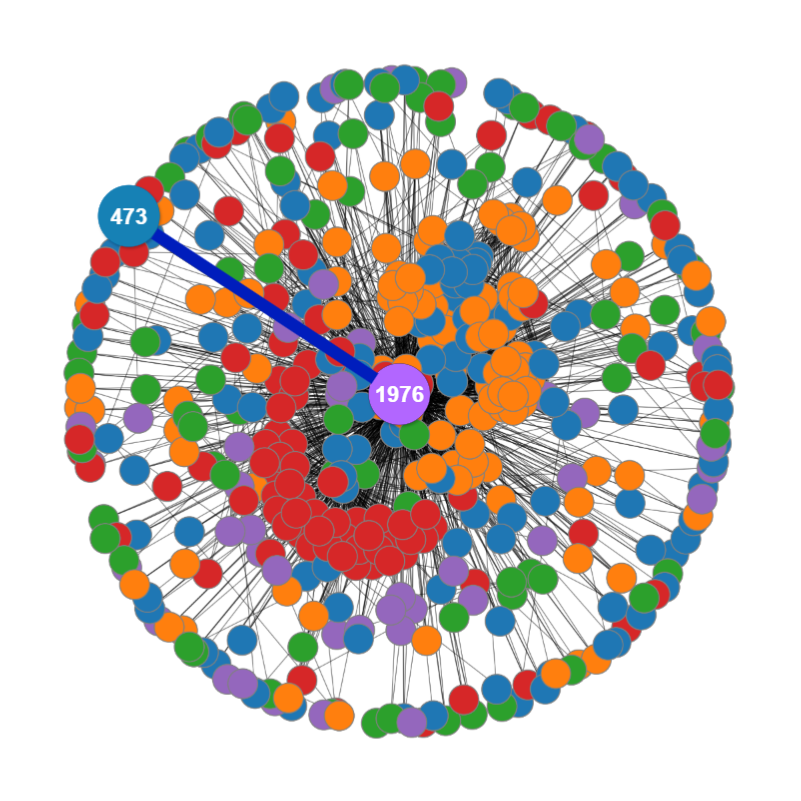}
        \caption{Chm$_{(1976,473)}$}
        \label{fig:bridge_mainG_chameleon_1976_473}
    \end{subfigure}%
    ~
    \begin{subfigure}{0.242\textwidth}
        \centering
        \includegraphics[width=\textwidth]{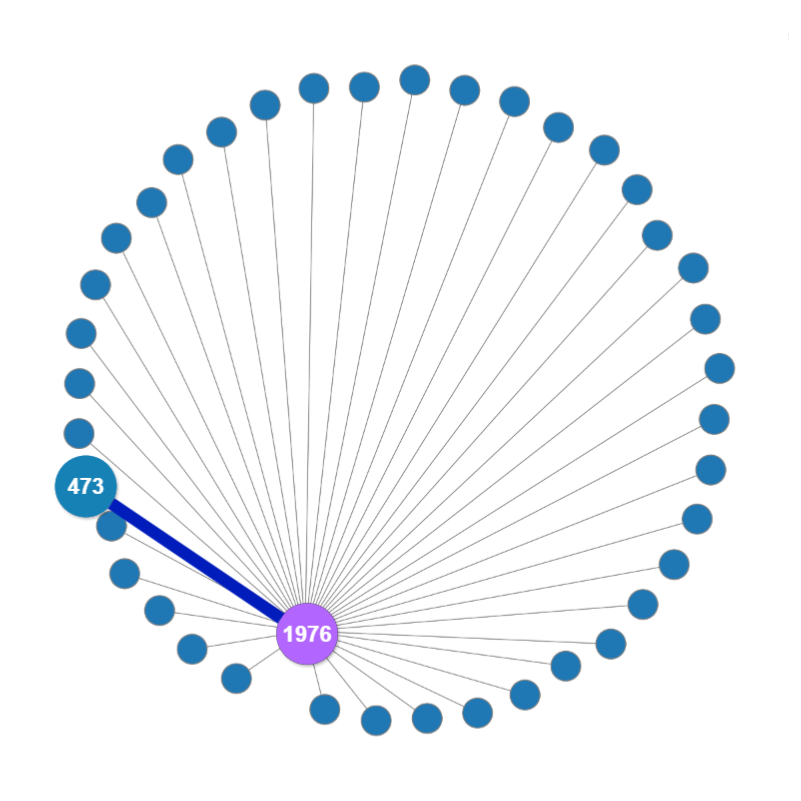}
        \caption{Chm-k${_{1}}_{(1976,473)}$}
        \label{fig:bridge_coreSubG_chameleon_1976_473}
    \end{subfigure}%
    ~
    \begin{subfigure}{0.242\textwidth}
        \centering
        \includegraphics[width=\textwidth]{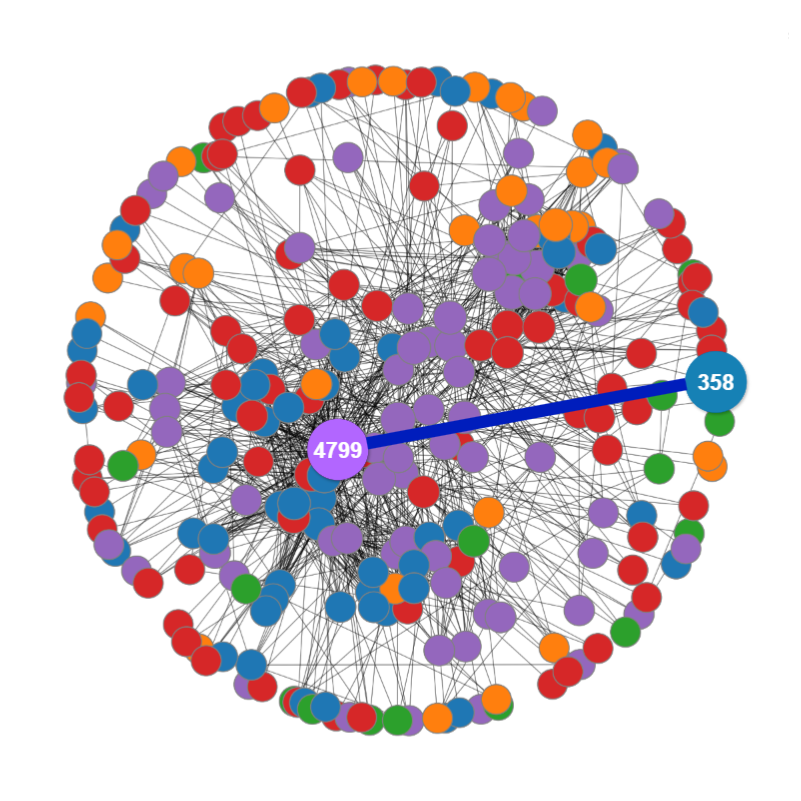}
        \caption{Sqr$_{(4799,358)}$}
        \label{fig:bridge_mainG_squirrel_4799_358}
    \end{subfigure}%
    ~
    \begin{subfigure}{0.242\textwidth}
        \centering
        \includegraphics[width=\textwidth]{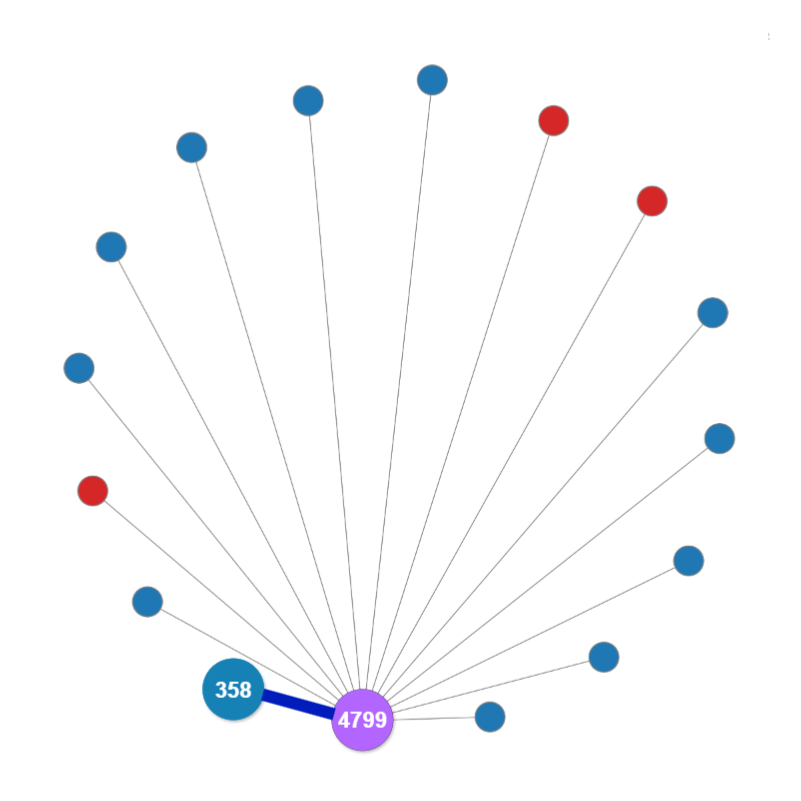}
        \caption{Sqr-k${_{1}}_{(4799,358)}$}
        \label{fig:bridge_coreSubG_squirrel_4799_358}
    \end{subfigure}%
    \caption{\textbf{Snippets of Bridge Analysis}. Edge $(1976,473)$ in Chameleon (Chm) and $(4799,358)$ in Squirrel (Sqr) datasets. $k_{1}$  denotes the subgraph $(S_1)$ and \textcolor{blue}{\rule{2em}{0.3em}} presents the Bridge Edges. 
    }

    \label{fig:extended_pilot_study} 
    \vspace{-6mm}
\end{figure*}

Figure~\ref{fig:extended_pilot_study}, illustrates the $2-hop$ surroundings of bridge edges in the both original graph (Figs. \ref{fig:bridge_mainG_chameleon_1976_473} and \ref{fig:bridge_mainG_squirrel_4799_358}) and the edge-induced partitioned subgraph (Figs. \ref{fig:bridge_coreSubG_chameleon_1976_473} and \ref{fig:bridge_coreSubG_squirrel_4799_358}). In the original graphs, bridge edges show strong heterophilic characteristics, causing cross-class contamination in GNN's aggregation and reducing representational expressivity. After $k$-core decomposition, the surroundings of bridge edges remain heterophilic and resemble star like pattern where hub node differs from its neighbors. However, most client nodes share the same class label. This allow a $2-layer$ GNN to effectively capture the homophilic structure. Hence, the partition  preserves latent homophily in those subgraphs that enhances model's performance. 
\begin{figure*}[t!]
    \centering
    \begin{subfigure}{0.33\textwidth}
        \centering
        \includegraphics[width=\textwidth]{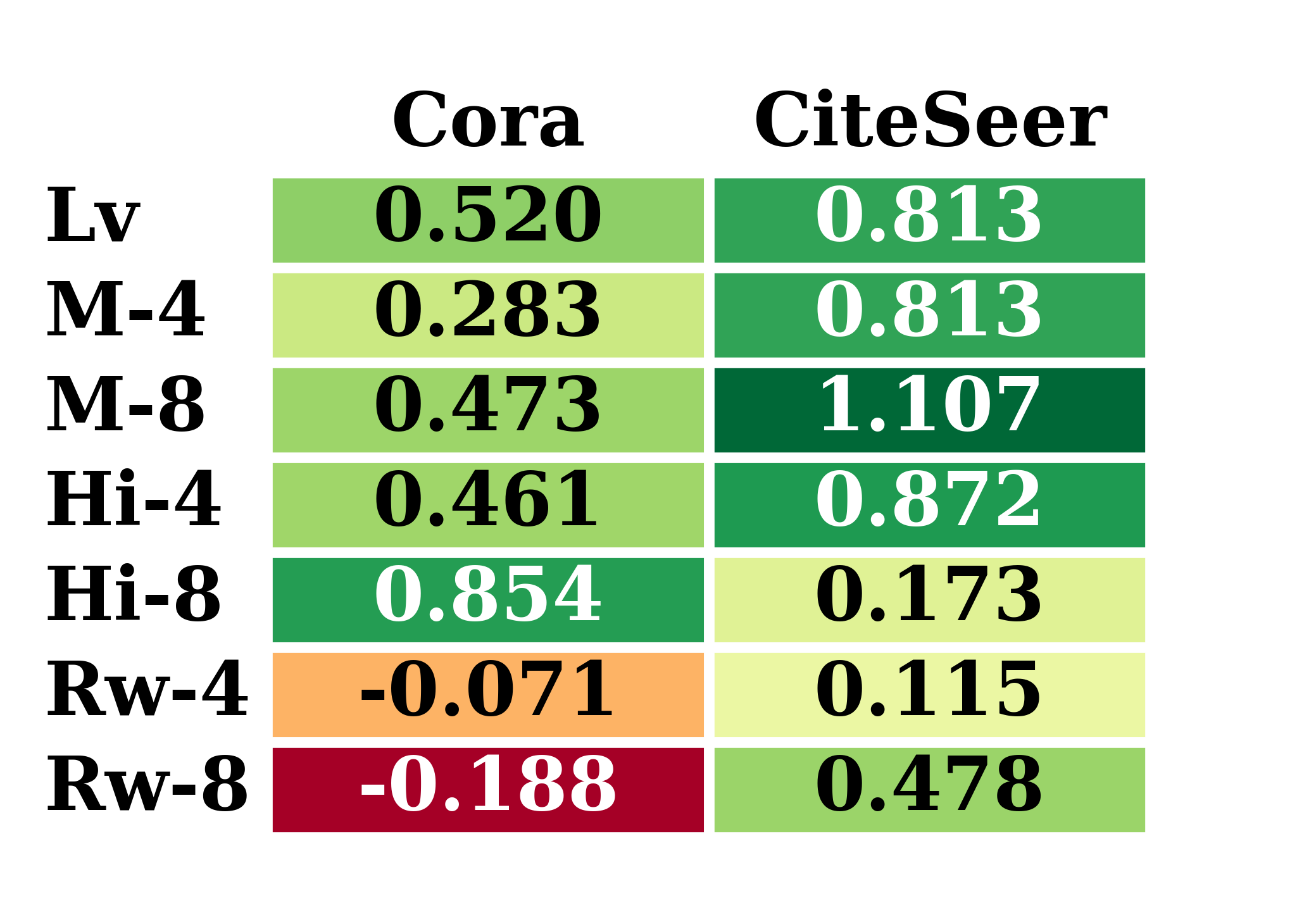}
        \caption{}
        \label{fig:cora-cs-other-decom}
    \end{subfigure}%
    ~
    \begin{subfigure}{0.33\textwidth}
        \centering
        \includegraphics[width=\textwidth]{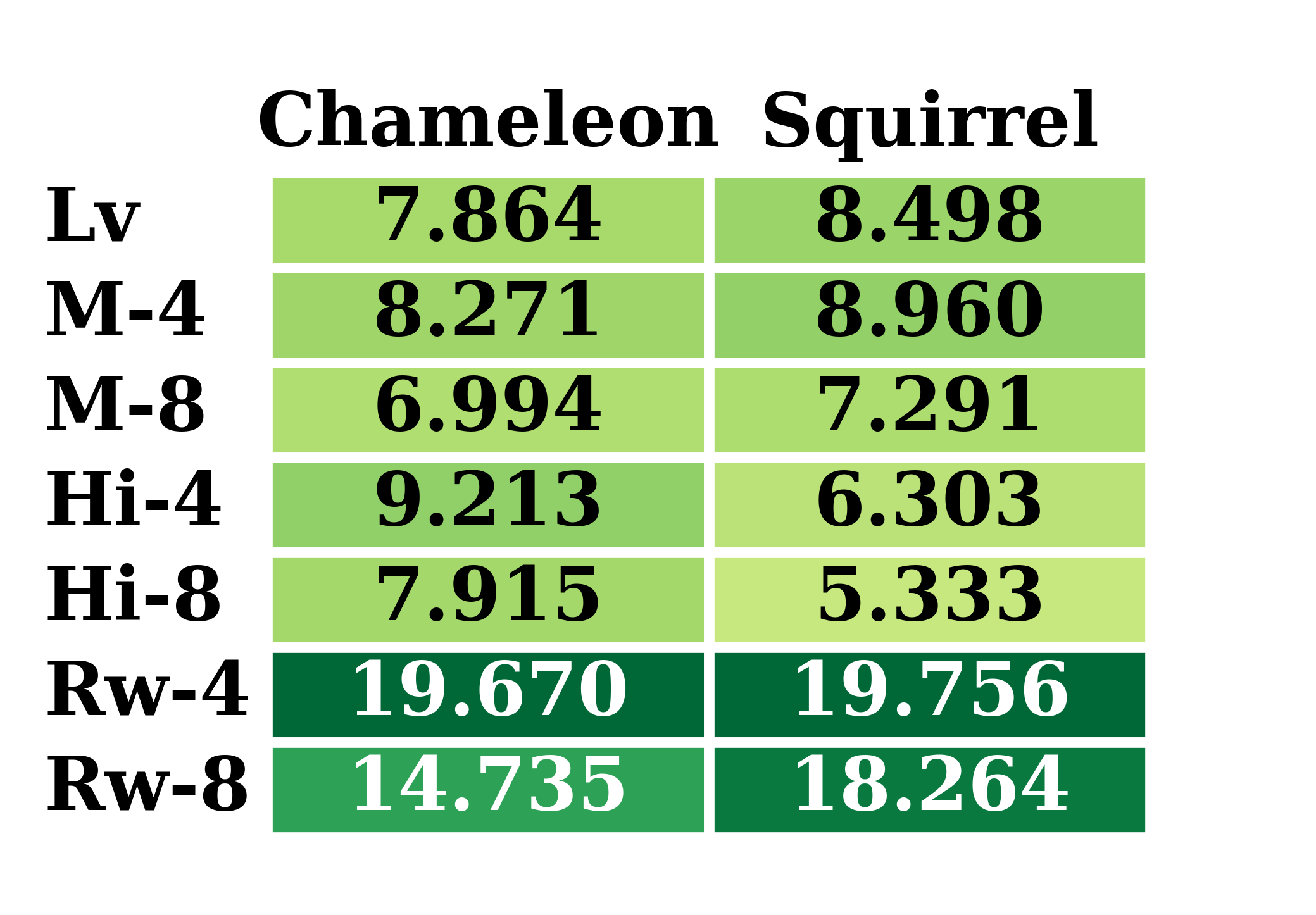}
        \caption{}
        \label{fig:chm-sqr-other-decom-ppr}
    \end{subfigure}%
   ~
    \begin{subfigure}{0.33\textwidth}
        \centering
        \includegraphics[width=\textwidth]{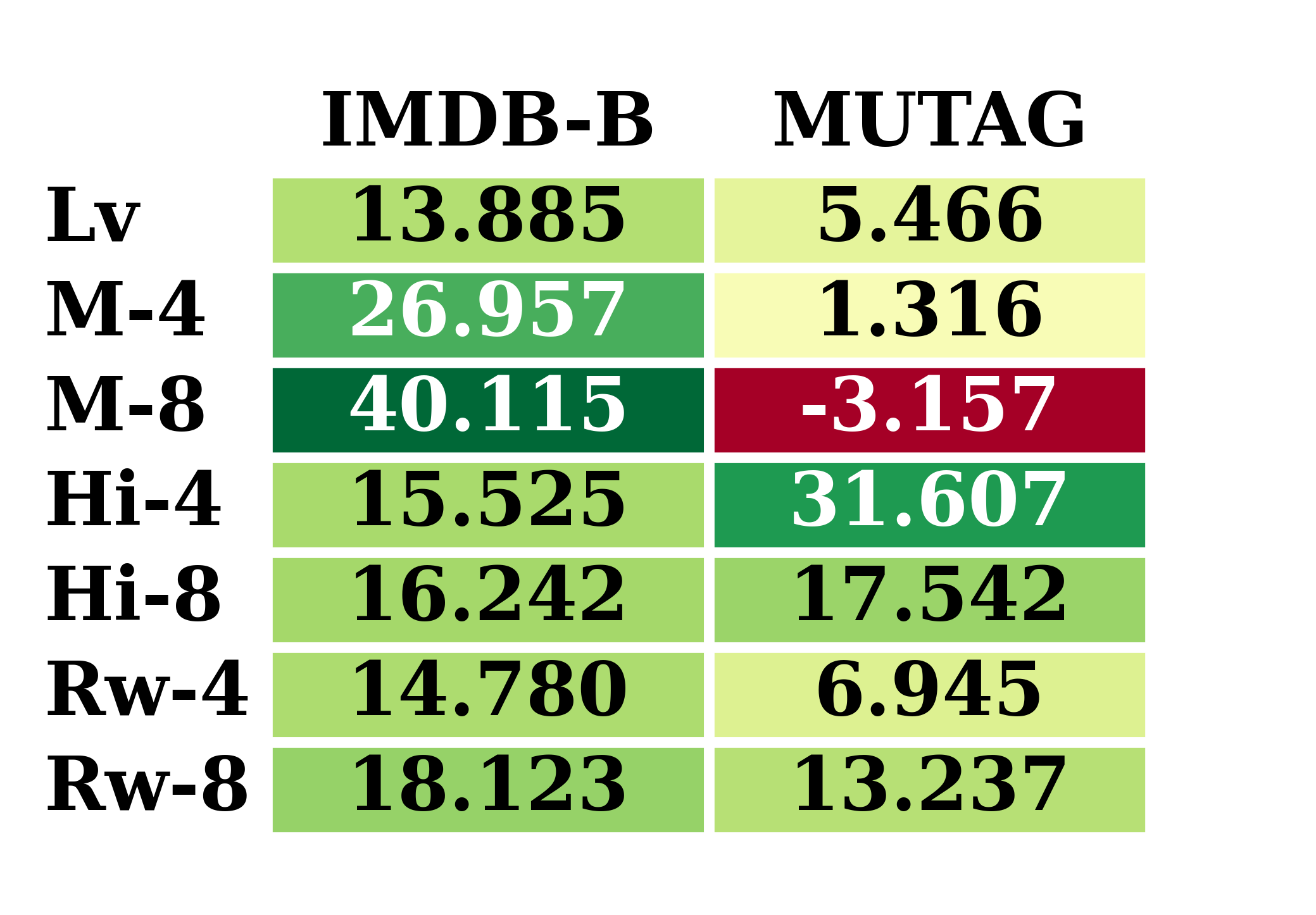}
        \caption{}
        \label{fig:imdb-mutag-other-decom}
    \end{subfigure}%
    \caption{ \textbf{Performance gain (in \%)$-\cacose$ vs other Decomp}ositions: Illustrate on six datasets (4 NC and 2 GC). Louvain $(Lv)$, Metis $(M)$, Hierarchical $(Hi)$, and Random-Walk $(Rw)$. Except for Louvain, the other methods are annotated with the number of partitions, i.e, $(M-4)$ presents Metis with $4$-partitions. Green and red shades present the positive and negative gains respectively.}

    \label{fig:compare_decom_methods}. 
    \vspace{-8mm}
\end{figure*}

\begin{table}[t!]
    \centering
    \caption{Ablation Study and Performance Comparison across Datasets. Accuracy values exceeding $\cacose$ are indicated in bold.}
    \label{tab:ablation-study}
    \begin{tabular}{@{}l |c| c c c | c c@{}}
        \hline
    \cellcolor{gray!15}    \textbf{Change} & \cellcolor{gray!15} \textbf{Component} & \cellcolor{gray!15} \textbf{Cora} & \cellcolor{gray!15} \textbf{C.Seer} & \cellcolor{gray!15} \textbf{Cham.} & \cellcolor{gray!15} \textbf{PROT.} & \cellcolor{gray!15} \textbf{IMDB-B} \\ \hline
        Attention & without & \multirow{2}{*}{84.42} & \multirow{2}{*}{68.38} & \multirow{2}{*}{\textbf{69.23}} & \multirow{2}{*}{\textbf{71.87}} & \multirow{2}{*}{73.00} \\
        Mechanism & cross-Attention & & & & & \\ \hline
        CaEF & without CaEF & 84.96 & 68.41 & 68.73 & 65.89 & 73.10 \\ \hline
        \multirow{3}{*}{SAGPool} & TopkPool & 84.94 & 69.20 & 65.62 & 65.09 & 72.90 \\
         & DMon & \textbf{85.10} & 69.37 & 65.82 & \textbf{74.55} & 67.50 \\
         & GMT & 84.70 & \textbf{69.74} & 65.27 & 56.90 & 60.71 \\ \hline
        \multirow{2}{*}{GCN} & GAT & 83.31 & 68.86 & 65.49 & 61.33 & 71.70 \\
         & SAGE & 83.92 & 69.12 & 60.04 & 63.84 & \textbf{73.90} \\ \hline
    \end{tabular}
\end{table}
\vspace{-4mm}
\subsubsection{Ablation Study.}
In this analysis, we experiment with our model by altering its components. Figure~\ref{fig:compare_decom_methods} illustrates the performance gains of the $k$-core decomposition in $\cacose$ compared to alternative partition methods on six different datasets. We use four NC datasets: two citations networks (Cora, CiteSeer) and two heterophilic networks (Chameleon, Squirrel). Besides, two GC datasets: IMDB-BINARY and MUTAG. With a few exceptions, our model consistently achieves higher accuracy than other decompositions. 

Table~\ref{tab:ablation-study} presents an ablation study on five datasets (three NC and two GC) to asses the contribution of different modules. First, we remove the attention mechanism, followed by the exclusion of the closure component ($CaEF$). In most cases, the fall in accuracy highlights the  significance of these components in $\cacose$. 
Next, we replace SAGPool with alternative pooling methods, including TopKPool~\cite{gao2019graph}, DMonPool~\cite{tsitsulin2023graph} and GMT~\cite{bacciu2021k}. With only a few exceptions, our model consistently outperforms these variants. Finally, we substitute the GCN backbone with GAT and GraphSAGE. Although GraphSAGE enables us to achieve higher accuracy on the IMDB-BINARY dataset, it performs worse on other datasets. Regarding GAT, our model consistently yields superior performance across all datasets.

\section{Conclusion}
In a nutshell, our $\cacose$ model successfully demonstrates its efficacy that facilitates graph representation learning  to overcome long-range dependency in GNNs. Through closure-aware cohesive graph decomposition, it provides essential locality to the vertices. Besides, the SAGPool filters out noisy connections in networks and the attention mechanism provides crucial global information across subgraphs. Extensive experiments on benchmark datasets illustrate its consistency over the standard graph learning models. We expect this technique will open up some new research directions: subgraph-wise graph rewiring, regional graph super-resolution and parallel subgraph learning to develop robust graph representation learning models.

\bibliographystyle{splncs04}
\bibliography{reference}

\newpage

\appendix
\section{Appendix} \label{sec:appendix}

\subsection{Proof through Ollivier-Ricci Curvature}

Olliver-Ricci curvature~\cite{nguyen2023revisiting}---denoted by $\kappa(u,v)$---is a measure of the local similarity of the neighborhoods of two vertices, typically by means of random walk measures centered at those vertices. This value is positive or negative when the overlap between the probability distributions of random walkers between $N(u)$ and $N(v)$ is large or small, respectively.

Let $G = \left(V_G, E_G \right)$ be a simple graph with shortest path distance $d(u,v)$ for all $u,v \in V_G$. Let $\Pi(\mu_u, \mu_v)$ be the family of joint probability distributions of $\mu_u$ and $\mu_v$, and take $\mu_u$ to be the $1$-step random walk measure given by

\begin{equation*}
\mu_u(v) = 
\begin{cases}
\frac{1}{deg(u)} & \text{if }v \in N(u) \\
0 & \text{otherwise.}
\end{cases}
\end{equation*}
Then define the $L^1$-Wasserstein distance as
$$ W_1 \left( \mu_u, \mu_v \right) = \inf_{\pi \in \Pi \left( \mu_u, \mu_v  \right)} \left( \sum_{(p, q) \in V_{G}^{2}} \pi(p,q) d(p,q) \right),$$
which measures the minimum distance random walks between $u$ and $v$ must take to meet each other. The Ollivier-Ricci curvature is then
$$ \kappa(u,v) = 1 - \frac{W_1(\mu_u, \mu_v)}{d(u, v)}. $$
This carries the interpretation that whenever two random walkers are unlikely to meet, we have $\kappa(u,v) < 0.$

\begin{theorem}\label{thm:ricci-curvature}
For all $(u,v) \in E_G,$ if $N(u) \cap N(v) = \emptyset,$ then $\kappa(u,v) \leq 0$
\end{theorem}

\begin{proof}
For any $\pi \in \Pi \left( \mu_u, \mu_v \right),$ consider the case of $\pi(p,q) > 0$ since $\pi(p,q) = 0$ contributes $0$ to $\sum_{(p, q) \in V_{G}^{2}} \pi(p,q) d(p,q).$ Then $p \in N(u)$ and $q \in N(v)$ since $\sum_{q \in V_G} \pi(p,q) = \mu_u(p)$ and $\sum_{p \in V_G} \pi(p,q) = \mu_v(q)$ give $p \notin N(u) \Rightarrow \sum_{q \in V_G} \pi(p,q) = 0$ and $q \notin N(v) \Rightarrow \sum_{q \in V_G} \pi(p,q) = 0,$ respectively. But $N(u) \cap N(v) = \emptyset,$ so $p \neq q$ for all $(p,q) \in V_{G}^2$ with $\pi (p,q) > 0.$ Thus, $d(p,q) \geq 1.$ It follows that
\begin{equation*}
\sum_{(p, q) \in V_{G}^{2}} \pi(p,q) d(p,q) \geq \sum_{(p, q) \in V_{G}^{2}} \pi(p,q) \cdot 1 = 1.
\end{equation*}
Moreover, $W_1 \left( \mu_u, \mu_v \right) \geq 1.$ Since $(u,v) \in E_G,$ we have $d(u,v) = 1.$ Hence,

$$ \kappa(u,v) = 1 - W_1(\mu_u, \mu_v) \leq 0.$$
\end{proof}

Thus, Theorem~\ref{thm:ricci-curvature} shows that the edges removed by the $CaEF$ in algorithm have non-positive Ollivier-Ricci curvature; in other words, they correspond with bottlenecks. This approach can then be said to mitigate oversquashing by removing the message passing pathways associated with this class of bottlenecks.

\begin{algorithm}[h!]
\SetAlgoVlined
\KwInput{A graph $G=(V,E)$}
\KwOutput{Vertex embeddings $Z_v$ and graph embedding $Z_G$}
\DontPrintSemicolon
\SetAlgoNoEnd

\tcc{Measure edge score with the $k$-core algorithm}
$C(u,v)=\max\{k \mid (u,v)\in G_k\}$\; \label{line_cacos:measure_edge_scores}

\tcc{Apply CaEF}
\If{$S(u,v)=0$}{
    $C(u,v)=C(u,v)-1$\; \label{line_cacos:apply_CaEF_and_update_score}
}

$S_k=\{(u,v)\in E : C(u,v)=k\}$\; \label{line_cacos:edge_induced_decom_starts}

\tcc{Extract subgraphs}
$S=\{\text{Edge-Subgraph}(S_k) : k\in 1,\dots,K_{\max}\}$\;

\ForEach{$S_k=(V_S,E_S)\in S$}{
    $H_k=\text{GCN}_k(S_k)$\;
    $Z_k=\text{SAGPool}(S_k,H_k)$\; \label{line_cacos:applied-gcn_and_pooling}
}

$Z_{S_k}^{attn}=\text{Attention}(Z_{S_k})$\; \tcc{Apply attention among subgraphs} \label{line_cacos:cross-attention}

$Z_G=\text{Mean}\left([Z_{S_1}^{attn}, Z_{S_2}^{attn}, \dots, Z_{S_k}^{attn}]\right)$\; \label{line_cacos:compute-graph_embedding}

$Z_v=\text{zeros}(N_v, d_v+d_S)$\; \label{line_cacos:node_embedding_start}

\ForEach{$(H_k,Z_k^{attn})$ associated with $S_k$}{
    \ForEach{$h_v\in H_k$}{
        $h_v=h_v \parallel Z_k^{attn}$\; \tcc{Concatenation operation}
        $z_v=z_v+h_v$\; \tcc{Map $v$ from $H_k$ to $Z_v$} \label{line_cacos:node_embedding_end}
    }
}

\Return $Z_G, Z_v$\;
\caption{CaCoS Algorithm}
\label{alg:cacos}
\end{algorithm}

\subsection{Representation Learning by $\cacose$}
Algorithm~\ref{alg:cacos} presents the execution of the $\cacose$ model for $k$-core decomposition. In lines (\ref{line_cacos:measure_edge_scores}-\ref{line_cacos:apply_CaEF_and_update_score}), at first, the input graph is decomposed with the $k$-core algorithm, where nodes achieve their coreness score. Next, $\cacose$ explores all the edges and assigns the coreness scores as the edge weight. Additionally, it utilizes edge filtering $(CaEF)$ to update specific narrow channels edge scores. In the next stage (lines \ref{line_cacos:edge_induced_decom_starts}-\ref{line_cacos:applied-gcn_and_pooling}), the $\cacose$ extracts edge-induced subgraphs from the input graph based on the coreness scores. Note that, the $k$-core algorithm~\cite{malliaros2020core} iteratively removes the nodes with degree less than $k$ as $k$ increases. Thus, in each node removal stage, the $(k-1)$ value is assigned to the removed nodes.

Then, on the partitioned subgraphs, it concurrently applies $GCN$ for node representation learning as well as $SAGPool$ to encode the entire subgraph's information. In the next step (line~\ref{line_cacos:cross-attention}), $\cacose$ utilizes the attention mechanism across the subgraphs' embeddings to capture the mutual information among them. Next, it takes the average of all the subgraph embeddings that represent the final representation of the graph (line~\ref{line_cacos:compute-graph_embedding}). In the context of node embeddings (lines \ref{line_cacos:node_embedding_start}-\ref{line_cacos:node_embedding_end}), each subgraph's cross-attentive embeddings are concatenated with its node representations, while the shared nodes are mapped and their combined features are summed up in the vertices' global feature matrix.

\subsubsection{Complexity.}
In worst case, the time complexities of $k$-core, SAGPool, and cross attention are $O(V + E)$, $O (V^{2})$, and $O(k^{2}d)$. Hence the overall complexity of $\cacose$ is $O(V^{2} + V+ E+ k^{2}d)$. Although looks complex, $\cacose$ benefits from parallelizable $k$-core decompositions and efficient GPU execution of SAGPool and cross-attention. The detailed algorithm is omitted due to space constraints.

\begin{figure*}[h!]
    \centering
    \includegraphics[width=\textwidth]{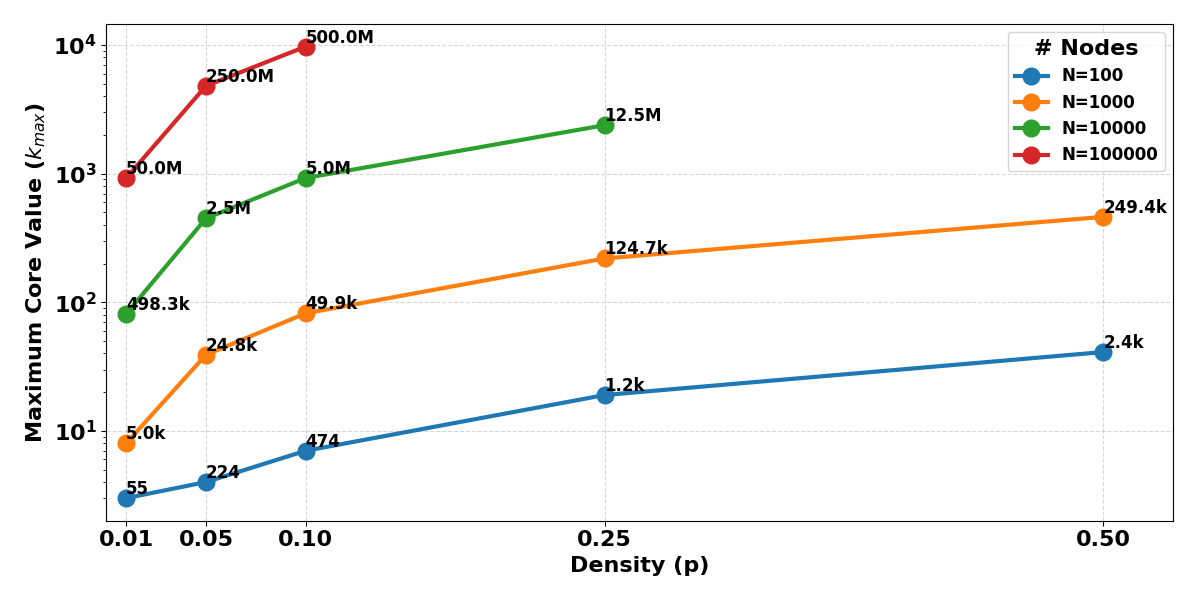}
    \caption{\textbf{Scalability Test.} This figure presents the maximum core value $K_{max}$ along the $y$-axis for networks generated with various combinations of graph size $(\# Nodes)$ and density $(p)$. The value of $K_{max}$ is corresponds to the maximum number of subgraphs that can be extracted through $k$-core decomposition.}
    \label{fig:scalability_test}
\end{figure*}

\vspace{-10mm}
\subsubsection{Scalability.}

In this experiment, we apply the Erd\H{o}s--R\'enyi graph generator to generate random networks of varying scale. During the generation process, we construct graphs with vertex count of $10^2$, $10^3$, $10^4$ and $10^5$. For each graph size ($\#Nodes$) we employ the edge creation probability $p \in \{0.01, 0.05, 0.10, 0.25, 0.50\}$. As the graph size increases, the number of possible edges-and consequently the total generated edges-grows significantly. 

For each generated graph, we compute the maximum core number $(K_{max})$m representing the highest number of subgraphs that can be extracted through the $k$-core decomposition. Figure~\ref{fig:scalability_test} plots $K_{max}$ against the graph's density $p$ ($x$-axis) for networks having different scales (distinguished by colors and markers). Besides, each plot is annotated with the number of edges corresponding to each combination of graph size and density. 

The result depicts that even in extremely large and densely connected networks, the maximum core value remains below $10^4$. In practical cases the value of $K_{max}$ is substantially smaller, this indicates that the number of extractable cohesive subgraphs via $k$-core decomposition is limited in realistic settings. 

Since, the real-world graphs are sparser and the resulting subgraphs sizes are smaller, it implies that the computational overhead introduced by the attention mechanism in our framework is negligible. It is worth noting that, the generation process was constrained by hardware capacity. For graphs with $10^4$ vertices we limit the density $p \le 0.25$, and for graph size with $10^5$, we consider $p\le 0.10$. 
\end{document}